\documentclass[table,xcdraw]{article}

\usepackage{arxiv}

\usepackage[utf8]{inputenc} %
\usepackage[T1]{fontenc}    %
\usepackage{hyperref}       %
\usepackage{url}            %
\usepackage{booktabs}       %
\usepackage{amsfonts}       %
\usepackage{nicefrac}       %
\usepackage{microtype}      %
\usepackage{lipsum}		%
\usepackage{graphicx}
\usepackage{natbib}
\usepackage{doi}
\usepackage{subfigure}
\usepackage{amsmath}
\usepackage{amsthm}
\usepackage{amsmath,amssymb,amsfonts,graphicx,nicefrac,mathtools,amsthm,xfrac}
\usepackage{fancyhdr}
\usepackage{extramarks}
\usepackage{bbm}
\usepackage{amsthm}
\usepackage{tikz}
\usepackage{enumerate}
\usepackage{multirow}
\usepackage{xcolor}
\usepackage{float}

\usepackage{hyperref, cleveref}
\usepackage{hyperref, cleveref}

\def\g{\gamma}
\def\o{\omega}
\def\a{\alpha}
\def\l{\ell}
\def\G{\Gamma}
\def\I{\text{IG}}

\def\max{\text{max}}
\def\E{\mathbb{E}} %

\def\R{\mathbb{R}}
\def\N{\mathbb{N}}
\def\F{\mathcal{F}}
\def\A{\mathcal{A}}
\def\Z{\mathbb{Z}}
\def\I{\text{IG}}
\def\EG{\text{EG}}
\def\pxi{\frac{\partial F}{\partial x_i}}

\theoremstyle{definition}

\newtheorem{theorem}{Theorem}
\newtheorem{conjecture}{Conjecture}
\newtheorem{claim}{Claim}
\newtheorem{lemma}{Lemma}
\newtheorem{corollary}{Corollary}

\newtheorem{definition}{Definition}

\title{A Rigorous Study of Integrated Gradients Method and Extensions to Internal Neuron Attributions}

\author{ {Daniel Lundstrom} \\
	University of Southern California\\
	\texttt{lundstro@usc.edu} \\
	\And
    {Tianjian Huang} \\
	University of Southern California\\
	\texttt{tianjian@usc.edu} \\
	\And
    {Meisam Razaviyayn} \\
	University of Southern California\\
	\texttt{razaviya@usc.edu} \\
}

\date{}

\renewcommand{\headeright}{}
\renewcommand{\undertitle}{}
\renewcommand{\shorttitle}{A Rigorous Study of Integrated Gradients Method and Extensions to Internal Neuron Attributions}

\begin{document}
\maketitle

\begin{abstract}

As deep learning (DL) efficacy grows, concerns for poor model explainability grow also. Attribution methods address the issue of explainability by quantifying the importance of an input feature for a model prediction. Among various methods, Integrated Gradients (IG) sets itself apart by claiming other methods failed to satisfy desirable axioms, while IG and methods like it uniquely satisfy said axioms. This paper comments on fundamental aspects of IG and its applications/extensions: 1) We identify key differences between IG function spaces and the supporting literature’s function spaces which problematize previous claims of IG uniqueness. We show that with the introduction of an additional axiom, \textit{non-decreasing positivity}, the uniqueness claims can be established. 2) We address the question of input sensitivity by identifying function classes where IG is/is not Lipschitz in the attributed input. 3) We show that axioms for single-baseline methods have analogous properties for methods with probability distribution baselines. 4) We introduce a computationally efficient method of identifying internal neurons that contribute to specified regions of an IG attribution map. Finally, we present experimental results validating this method.\footnote{This work was supported in part with funding from the USC-Meta Center for Research and Education in AI \& Learning (REAL@USC center).}
\end{abstract}

\section{Introduction}

Deep neural networks have revolutionized the field of vision processing, showing marked accuracy for varied and large scale computer vision tasks \citep{huang2017densely}, \citep{ren2016faster}, \citep{bochkovskiy2020yolov4}. At the same time, deep neural networks suffer from a lack of interpretability. Various methods have been developed to address the interpretability problem by quantifying, or attributing, the importance of each input feature to a model's output. Basic techniques inspect the gradient of the output with respect to an input~\citep{baehrens2010explain}. Deconvolutional networks~\citep{zeiler2014visualizing} employ deep networks to produce attributions, while guided back-propagation~\citep{springenberg2014striving} gives attributions to internal neuron activations. Methods such as Deeplift~\citep{shrikumar2017learning} and Layer-wise relevance propagation~\citep{binder2016layer} employ a baseline to use as a comparison to the input (called baseline attributions). Further methods include \citet{zhou2016learning}, \citet{zintgraf2016new}. %

\citet{sundararajan2017axiomatic} introduced the baseline attribution method of Integrated Gradients (IG). The paper identified a set of desirable axioms for attributions, demonstrated that previous methods fail to satisfy them, and introduced the IG method which satisfied the axioms. Included was the claim that any method satisfying a subset of the axioms must be a more general form of IG (called path methods).

\noindent\textbf{Contributions.} This paper addresses multiple aspects of the IG method: its foundational claims, mathematical behavior, and extensions. The IG paper of~\citet{sundararajan2017axiomatic} applies results from \citet{friedman2004paths} (given here as Theorem~\ref{thm1}) to claim that path methods (defined below) are the only methods that satisfy a set of desirable axioms. Upon inspection, we observe that there are key assumptions of the function spaces of \citet{friedman2004paths}, such as functions being non-decreasing, which are not true in the DL context. These differences in function spaces were unaddressed in \citet{sundararajan2017axiomatic}. We show that because the function spaces differ, Theorem~\ref{thm1} does not apply and the uniqueness claim is false. This observation also invalidates other uniqueness claims found in \citet{xu2020attribution} and \citet{sundararajan2020many}. With the introduction of an additional axiom, \textit{non-decreasing positivity} (NDP), we show that Theorem~\ref{thm1} can apply, and rigorously extend it into a broad-ranging DL function space.

We address the mathematical behavior of IG and an extension. We identify a common class of functions where IG may be hypersensitive to the input image by failing to be Lipschitz continuous, as well as a function class where IG is guaranteed to be Lipschitz continuous. We also note that the axioms in~\citet{sundararajan2017axiomatic} apply to single baseline attribution methods, but no such axioms have been stated for methods that employ a distribution of baselines. We identify/extend axioms for the distribution of baselines methods that parallel those in the single baseline case.

Lastly, we introduce a computationally efficient method of attributing to an internal neuron it's contribution to a region of the IG map. If an IG map indicated certain regions or sub-features are important to an output, this method provides a means of inspecting which individual neurons are responsible for that region or sub-feature. %

\section{Background}

\subsection{Attribution Methods and Integrated Gradient}
For $a,b\in\R^n$, define $[a,b]$ to be the hyperrectangle with $a,b$ as opposite vertices. An example is a greyscale image, which would have $a=0, b$ be the vectorized pixel values of a white image, and $n$ be the pixel count. We denote a class of functions $F:[a,b]\rightarrow \R$ by $\F(a,b)$, or $\F$ if $a,b$ may be inferred from the context. An example is a single output of a DL  model classifying  images in a classification task. A \textit{baseline attribution method} (BAM) is defined as follows:
\begin{definition}[Baseline Attribution Method]\footnotemark{}
Given $x,x'\in [a,b]$, $F\in \F(a,b)$, a baseline attribution method is any function of the form $A:[a,b]\times[a,b]\times \F(a,b)\rightarrow\R^n$.
\end{definition}

We may drop $x'$ and write $A(x,F)$ if the baseline is fixed, or may be inferred. An attribution can be interpreted as assigning values to each input $x_i$ indicating $x_i$'s contribution to the model output, $F(x)$. Obviously, many BAMs in the function class do not practically track an input's contribution to an output. By considering properties desirable to an attribution method, we may restrict the function space further to those which more effectively track input contributions. 
To this end, let us define a \textit{path function} as follows:

\begin{definition}[Path Function]
A function $\g(x,x',t):[a,b]\times[a,b]\times[0,1]\rightarrow [a,b]$ is a path function if, for fixed $x,x'$, $\g(t):=\g(x,x',t)$ is a continuous, piecewise smooth curve from $x'$ to $x$.
\end{definition}

\footnotetext{It is possible to widen the definition of baseline attribution methods to include the model's implementation, not just the input and output. We use our definition for the scope of the paper.}

We may drop $x'$ or both $x,x'$ when they are fixed, and write $\g(x,t)$, $\g(t)$ respectively. If we further suppose that $\pxi(\g(t))$ exists almost everywhere\footnotemark{}, then the \textit{path method} associated with $\g$ can be defined as:
\begin{definition}[Path Method] Given the path function~$\g(\cdot,\cdot,\cdot)$, the corresponding path method is defined as
\small
\begin{equation}\label{eq:PathMethod}
    A^\g(x,x',F) = \int_0^1 \frac{\partial F}{\partial x_i} (\g(x,x',t)) \times \frac{\partial \g_i}{\partial t} (x,x',t) dt,
\end{equation}
\normalsize
where $\g_i$ denotes the $i$-th entry of $\g$.
\end{definition}
By definition, all path methods are baseline attribution methods. A \textit{monotone path method} is a path attribution where each path is monotone, i.e., $\g_i(t)$ is monotone in $t$  for all $i$. %
\footnotetext{A function exists almost everywhere if the set of points where the function is not defined has Lebesgue measure 0.}

The \textit{integrated gradient} method is a path method where the path is a straight line from $x'$ to $x$. Formally, choosing the monotone path $\g(t) = x' + t(x-x')$ yields the IG formula: 

\begin{definition}[Integrated Gradient Method\footnotemark{}]%
Given $x,x'\in [a,b]$, and $F\in \F(a,b)$, the integrated gradient attribution of the $i$-th component of $x$ is defined as
\begin{equation} \label{eq:IG} 
    \text{IG}_i(x,x',F) = (x_i-x_i')\int_0^1 \frac{\partial F}{\partial x_i} (x' + t(x-x')) dt
\end{equation}
\normalsize
\end{definition}
\citet{sundararajan2017axiomatic} uses a black baseline, i.e., $x'=0$. IG corresponds to the Aumann-Shaply  method in the cost-sharing literature~\cite{aumann1974values}. 

\footnotetext{Practically speaking, IG is relatively easy to implement. The IG is calculated by numerical integration with a recommended 20 to 300 calls of the gradient~\citep{sundararajan2017axiomatic}.}

\subsection{What makes IG unique?}
The theoretical allure of IG stems from three key claims: 1) IG satisfies stipulated axioms (desirable properties), 2) other  methods fail at least one of the  axioms, and 3) only methods like it (path methods) are able to satisfy these axioms. We will review the stated axioms in~\citet{sundararajan2017axiomatic}, and a reader can find an explanation of each axiom in appendix~\ref{ExplanationOfAxioms}. Let $A$ be a BAM, $x,x'\in [a,b]$, $F,G\in \F$. Then the axioms are as follows:

\begin{enumerate}
    \item \textit{Sensitivity(a)}: Suppose $x,x'$ vary in one component, so that $x_i \neq x_i'$, and $x_j = x_j'$ $\forall j\neq i$. Further suppose $F(x)\neq F(x')$. Then $A_i(x,x',F) \neq 0$.
    \item \textit{Implementation Invariance}: $A$ is not a function of model implementation, but solely of the mathematical mapping of the model's domain to the range.
    \item \textit{Completeness}: $\forall F\in\F$, $x, x'\in [a,b]$, we have: $\sum_{i=1}^{n} A_i(x,x',F) = F(x)-F(x')$.
    \item \textit{Linearity}: For $\alpha,\beta \in \R$, we have: $A_i(x,x',\alpha F+\beta G) = \alpha A_i(x,x',F) + \beta A_i(x,x',G)$.
    \item \textit{Sensitivity(b)/Dummy}: $\forall F\in \F$, if $\partial_i F \equiv 0$, then $A_i(x,x',F) = 0$.
    \item \textit{Symmetry Preserving}: For a given $(i,j)$, define $x^*$ by swapping the values of $x_i$ and $x_j$. Now suppose that $\forall x\in [a,b]$, $F(x) = F(x^*)$. Then whenever $x_i=x_j$ and $x'_i=x'_j$, we have $A_i(x,x',F) = A_j(x,x',F)$.
\end{enumerate}

The argument in~\citet{sundararajan2017axiomatic} is roughly as follows: other established methods fail to satisfy sensitivity(a) or implementation invariance. IG satisfies: completeness, a stronger claim that includes sensitivity(a); implementation invariance; linearity; and sensitivity(b). It can be shown that path methods are the unique methods that satisfy implementation invariance, sensitivity(b), linearity, and completeness. IG is the unique path method that satisfies symmetry. Thus, IG uniquely satisfies axioms 1-6. It was admitted that the Shaply-Shubik method~\citep{shapley1971assignment} also satisfy these conditions, but it is computationally infeasible.

It should be noted that \citep{lerma2021symmetry} pointed out that other computationally feasible path methods (single path methods) satisfying all axioms exist and are easy to produce, although they are not as simple as IG. It should also be noted that other axiomatic treatments of IG exist. \citet{sundararajan2020many} introduced an alternative set of axioms and claimed that IG uniquely satisfied them. \citet{xu2020attribution} claimed that path methods uniquely satisfy linearity, dummy, completeness, and an additional axiom. These treatments will be discussed later.

\subsection{Modifications and Extensions}\label{sectionMod&Extend}

One issue with IG is the \underline{noisiness of the attribution}. Sharp fluctuations in the gradient, sometimes called the shattered gradient problem \citep{balduzzi2017shattered}, are generally blamed. Another issue with integrated gradients is \underline{baseline choice}. If the baseline is a black image, then the $(x_i-x_i')$ term will be zero for any black pixel in the input image, causing those attributions to be zero. This is an issue if the black input pixels do contribute to image recognition, such as a model identifying an image of a blackbird.

A category of fixes to these issues rely on modifying the choice of input and baseline. \citet{smilkov2017smoothgrad} addresses the noisiness issue by introducing noise into the input and taking the average IG. A tutorial on~\citet{tensorflowIG} addresses baseline choice by averaging the results when using a white and black baseline. \citet{erion2021improving} claims synthetic baselines such as black and white images are out of distribution data points, and suggests using training images as baseline and taking the average. \citet{VisualizingBaselines} investigates various fixed and random baselines using blurring, Gaussian noise, and uniform noise.

Another category of fixes modifies the IG path. \citet{kapishnikov2021guided} identifies accumulated noise along the IG path as the cause of noisy attributions, and employs a guided path approach to reduce attribution noise. \citet{xu2020attribution} is concerned with the introduction of information using a baseline, and opts to use a path that progressively removes Gaussian blur from the attributed image.

Some IG extensions employ it for tasks beyond input attributions. \citet{dhamdhere2018important} calculates accumulated gradient flow through neurons to produce an internal neuron attribution method called conductance. \citet{shrikumar2018computationally} later identifies conductance as an augmented path method and provides a computational speedup. \citet{lundstrom2022explainability} compares a neuron's average attribution over different image classes to characterize neurons as class discriminators. \citet{erion2021improving} incorporates IG attributions in a regularization term during training to improve the quality of attributions and model robustness.

\section{Remarks on Original Paper and Other Uniqueness Claims}
\subsection{Remarks on Completeness, Path Definition}
We first address a few claims of the original IG paper \citep{sundararajan2017axiomatic} to add mathematical clarifications.  \citet[Remark 2]{sundararajan2017axiomatic}  states:

\begingroup
\addtolength\leftmargini{-0.18in}
\begin{quote}
\it
``Integrated gradients satisfies Sensivity(a) because Completeness implies Sensivity(a) and is thus a
strengthening of the Sensitivity(a) axiom. This is because
Sensitivity(a) refers to a case where the baseline and the
input differ only in one variable, for which Completeness
asserts that the difference in the two output values is equal
to the attribution to this variable."
\end{quote}
\endgroup

To clarify, completeness implies sensitivity(a) for IG, and for monotone path methods in general. The form of IG guarantees that any input that does not differ from the baseline will have zero attribution, due to the $x_i-x_i'$ term in \eqref{eq:IG}. If only one input differs from the baseline, and $F(x) \neq F(x')$, then the value $F(x)-F(x')\neq 0$ will be attributed to that input by completeness. However, completeness does not imply sensitivity(a) for general attribution methods, or for non-monotone path methods specifically.

In  \citet{sundararajan2017axiomatic}, monotone path methods (what they simply term path methods) are introduced as a generalization of the IG method. The section reads:

\begingroup
\addtolength\leftmargini{-0.18in}
\begin{quote}
\it ``Integrated gradients aggregate the gradients along the inputs that fall on the straightline between the baseline and the input. There are many other (non-straightline) paths that monotonically interpolate between the two points, and each such path will yield a different attribution method. For instance, consider the simple case when the input is two dimensional. Figure 1\footnotemark{} has examples of three paths, each of which corresponds to a different attribution method.

Formally, let $\g = (\g_1, . . . , \g_n) : [0, 1] \rightarrow \R^n$ be a smooth
function specifying a path in $\R^n$ from the baseline $x'$ to the input $x$, i.e., $\g(0) = x'$, and $\g(1) = x$."
\end{quote}
\endgroup

\footnotetext{See Figure~\ref{fig:classExamples}, Appendix~\ref{AAfig1}.}

By the referred figure, $P_1$ is identified as a path, but it is not smooth. It is simple enough to interpret smooth here to mean piecewise smooth. Note  that monotonicity is mentioned, and all examples in Figure 1 are monotone, but monotonicity is not explicitly included in the formal definition. The cited source on path methods, \citet{friedman2004paths}, only considers monotone paths. Thus, we assume that~\citet{sundararajan2017axiomatic} only considers monotone paths. The alternative is addressed in the discussion of Conjecture~\ref{conj1}.

\subsection{On \citet{sundararajan2017axiomatic}'s Uniqueness Claim}

In the original IG paper, an important uniqueness claim is given as follows \citep[Prop 2]{sundararajan2017axiomatic}:

\begingroup
\addtolength\leftmargini{-0.18in}
\begin{quote}
\it ``\citep{friedman2004paths} Path methods are the only attribution methods that always satisfy Implementation Invariance, Sensitivity(b), Linearity, and Completeness."
\end{quote}
\endgroup

The claim that every method that satisfies certain axioms must be a path method is an important claim for two reasons: 1) It categorically excludes every method that is not a path method from satisfying the axioms, and 2) It characterizes the form of methods satisfying the axioms. However, no proof of the statement is given, only the following remark (Remark 4):

\begingroup
\addtolength\leftmargini{-0.18in}
\begin{quote}
 \it   %
``Integrated gradients correspond to a cost-sharing method called Aumann-Shapley \citep{aumann1974values}. Proposition 2 holds
for our attribution problem because mathematically the
cost-sharing problem corresponds to the attribution problem with the benchmark fixed at the zero vector."
\end{quote}
\endgroup

The cost-sharing problem does correspond to the attribution problem with benchmark fixed at zero, \textit{with some key differences}. To understand the differences, we review the cost-share problem and  results in \citet{friedman2004paths}, then rigorously state  \citet[Proposition~2]{sundararajan2017axiomatic}. We will then point out discrepancies between the function spaces that make the application of the results in \citet{friedman2004paths} neither automatic nor, in one case, appropriate.

In \citet{friedman2004paths}, attributions are discussed within the context of the cost-sharing problem. Suppose $F$ gives the cost of satisfying the demands of various agents, given by $x$. Each input $x_i$ represents an agent's demand, $F(x)$ represents the cost of satisfying all demands, and the attribution to $x_i$ represents that agent's share in the total cost. It is assumed $F(0) = 0$, naturally, and because increased demands cannot result in a lower total cost, $F(x)$ is non-decreasing in each component of $x$. Furthermore, only $C^1$ cost functions are considered. To denote these restrictions on $F$ formally, we write that for a positive vector $a\in \R_+^n$, the set of attributed functions for a cost-sharing problem is denoted by $\F^0 = \{F\in\F(a,0)|F(0) = 0, F \in C^1, F \text{ non-decreasing in each component}\}$. There are also restrictions on attribution functions. The comparative baseline in this context is no demands, so $x'$ is fixed at 0. Because an agent's demands can only increase the cost, an agent's demands should only have positive cost-share. Thus cost-shares are non-negative. Formally we denote the set of baseline attributions in \citet{friedman2004paths} by $\A^0 = \{A:[a,0] \times \F^0 \rightarrow \R_+^n \}$.

Before we continue, we must define an ensemble of path methods. Let $\G(x,x')$ denote the set of all path functions projected onto their third component, so that $x,x'$ are fixed and $\g\in\G(x,x')$ is a function solely of $t$. %
We may write $\G(x)$ when $x'$ is fixed or apparent. Define the set of monotone path functions as $\G^m(x,x'):=\{\g \in \G(x,x') | \g \text{ is monotone in each component}  \}$. We can then define an \textit{ensemble of path methods}:
\begin{definition}\label{eq:ensemble}
A BAM $A$ is an ensemble of path methods if there exists a family of probability measures indexed by $x,x'\in[a,b]$, $\mu^{x,x'}$, each on $\G(x,x')$, such that:
\begin{equation}
    A(x,F) = \int_{\g \in \G(x)} A^{\g}(x,F) d\mu^x(\g)
\end{equation}
\end{definition}

An ensemble of path methods is an attribution method where, for a given $x'$, the attribution to $x$ is equivalent to an average among a distribution of path methods, regardless of $F$. This distribution depends on the fixed $x'$ and choice of $x$. If we only consider monotone paths, then we say that a BAM A is an ensamble of monotone path methods, and swap $\G(x)$ for $\G^m(x)$.

We now present Friedman's characterization theorem:

\begin{theorem}{(\citet[Thm 1]{friedman2004paths})}\label{thm1}

The following are equivalent:
\begin{enumerate}
    \item $A\in \A^0$ satisfies completeness, linearity\footnotemark{}, and sensitivity(b).
    \item $A\in \A^0$ is an ensemble of monotone path methods.
\end{enumerate}
\end{theorem}

\footnotetext{\citet{friedman2004paths} uses a weaker form of linearity: $A(x,F+G) = A(x,F) + A(x,G)$.}

To rigorously state \citet[Prop 2]{sundararajan2017axiomatic}, we must interpret the claim: ``path methods are the only attribution methods that always satisfy implementation invariance, sensitivity(b), linearity, and completeness." By ``path methods", \citet{sundararajan2017axiomatic} cannot exclude ensembles of path methods. Simply stated: if some path methods satisfy the axioms, then some ensembles of path methods, such as finite averages, satisfy the axioms also. Neither can it mean non-monotone path methods, since Theorem~\ref{thm1} only addresses monotone path methods, and, supposedly, the theorem applies immediately. Thus we will interpret ``path methods" as in Theorem~\ref{thm1}, as an ensemble of monotone path methods. Define $\F^\text{D}$ to be the set of DL models where one output is considered, and define $\A^\text{D}$ to be the set of attribution methods defined on $\F^\text{D}$. We now state the characterization theorem  in \citet{sundararajan2017axiomatic}:

\begin{claim}\label{Claim:FalsePathThm}
{(\citet[Prop 2]{sundararajan2017axiomatic})}\label{claim1}
Suppose $A\in \A^\text{D}$ satisfies completeness, linearity, sensitivity(b), and implementation invariance. Then for any fixed $x'$, $A(x,x',F)$ is an ensemble of monotone path methods. 
\end{claim}

As stated previously, there are several discrepancies between the function classes of Theorem~\ref{thm1} and Claim~\ref{Claim:FalsePathThm}. $F\in \F^\text{D}$ need not be non-decreasing nor $C^1$. $x'$ need not be $0$, and $F(x')$ has no restrictions. Additionally, attributions in $\A^{\text{D}}$ can take on negative values while those in $\A^0$ can not. The differences between $\F^0$ and  $\F^\text{D}$, $\A^0$ and  $\A^\text{D}$ make the application of Theorem~\ref{thm1} problematic in the DL context. In fact, Claim~\ref{Claim:FalsePathThm} is actually false. 

Note that monotone and non-monotone path methods satisfy completeness\footnotemark{}, linearity, sensitivity(b), and implementation invariance. Fixing the baseline to zero and $[a,b]=[0,1]^n$, there exists a non-monotone path $\o(t)$ and non-decreasing $F$ s.t. $A^\o(x,x',F)$ has negative components. However, if path $\g(t) = \g(x,x',t)$ is monotone and $F$ is non-decreasing, $\frac{\partial \g_i}{\partial t}\geq 0$ and $\pxi \geq0$, $\forall i$. By Eq.~\ref{eq:PathMethod}, $A^\g(x,x',F)\geq 0$ for monotone $\g$ and non-decreasing $F$, implying any ensemble of monotone path methods would be non-negative. Thus, $A^\o$ is not an ensemble of monotone path methods. For a full proof, see Appendix~\ref{claim1ce}.

\footnotetext{Path methods satisfy completeness because $\sum_i A^{\g}_i (x,x',F) = \int_0^1 \nabla F * d\g  = F(x)- F(x')$ by the fundamental theorem for line integrals.}

Why did this happen? Note that in the context of Theorem~\ref{thm1}, this counterexample is disallowed. $\A^0$ only includes attributions that give non-negative values. Non-monotone path methods can give negative values for functions in $\F^0$, so they are disallowed. However, what is excluded in the game-theoretic context is allowed in the DL context: $\F^D$ functions can increase or decrease from the their baseline, so by completeness, negative and positive attributions must be included. Thus, non-monotone path methods are not prohibited, they are fair game. Without additional constraints, this implies that non-monotone path methods are allowed.

The above example shows that the set of BAMs satisfying axioms 3-5 cannot be characterized as an ensemble of path methods over $\G^m$. Since the counter example was a non-monotone path method, perhaps the set BAMs can be characterized as an ensemble of path methods over $\G$.

\begin{conjecture}\label{conj1}
The following are equivalent:
\begin{itemize}
\item $A\in\A^\text{D}$ satisfies completeness, linearity, sensitivity(b), and implementation invariance.
\item For a fixed $x'$, $A\in\A^\text{D}$ is equivalent to an ensemble of path methods where the maximal path length of the support of $\mu^x$ is bounded.\footnotemark{}
\end{itemize}
\end{conjecture}
\footnotetext{For the necessity of bounding maximal path length, see the Appendix~\ref{conj1comm}}
If Conjecture~\ref{conj1} were true, it would somewhat preserve the intention of Claim~\ref{claim1}: that BAMs satisfying axioms 3-5 are path methods. However, it is not clear how Theorem~\ref{thm1} can be used to support Conjecture~\ref{conj1}, since it proves characterizations exclusively with monotone path ensembles. On the other hand, it is an open question whether conjecture~\ref{conj1} is false, that is, perhaps there is a BAM satisfying axioms 3-5 that is not an ensemble of path methods. 

Even if we do not have any path characterization for BAMs satisfying axioms 3-5, we submit an insight into BAMs satisfying axioms 4 and 5.
\begin{lemma}\label{lemma1}
Suppose a BAM $A$ satisfies linearity and sensitivity(b), and $\nabla F$ is defined on $[a,b]$. Then $A(x,x',F)$ is a function solely of $x,x'$, and the gradient of $F$. Furthermore, $A_i(x,x',F)$ is a function solely of $x,x'$ and $\frac{\partial F}{\partial x_i}$ .
\end{lemma}

\subsection{On Other Uniqueness Claims}\label{OtherUniquenessClaims}
There are other attempts to establish the uniqueness of IG or path methods by referencing cost-sharing literature, each of which succumbs to the same issue as Claim~\ref{Claim:FalsePathThm}. The claims make use of an additional axiom, Affine Scale Invariance (ASI).\footnotemark{} We denote the function composition operator by "$\circ$". The ASI axiom, seventh in our list, is as follows:
\begin{enumerate}
    \setcounter{enumi}{6}
    \item\textit{Affine Scale Invariance (ASI)}: For a given index $i$, $c\neq 0$, $d$, define the affine transformation $T(x) := (x_1,...,cx_i+d,...,x_n)$. Then whenever $x, x', T(x), T(x')\in [a,b]$, we have $A(x,x',F) = A(T(x),T(x'),F  \circ  T^{-1})$.
\end{enumerate}

\footnotetext{\citet{xu2020attribution} and \citet{sundararajan2017axiomatic} gives an incorrect definition of ASI, saying $A(x,x',F) = A(T(x),T(x'),F  \circ  T))$. The source definition is from \citet{friedman1999three}.}

\citet[Prop 1]{xu2020attribution} claims that path methods are the unique methods that satisfy dummy, linearity, completeness, and ASI. Here the situation is similar to \citet{sundararajan2017axiomatic}: they import game-theoretic results from \citet{friedman2004paths} which assumes functions are non-decreasing and attributions are non-negative. As mentioned in our discussion of claim~\ref{Claim:FalsePathThm}, the referenced result can not be correctly applied to the context where attributions can be negative and no additional constraints are imposed. For a fuller treatment and counterexample, see Appendix~\ref{OtherClaimsCE}.

In another paper, \citet[Cor 4.4]{sundararajan2020many} claims that IG uniquely satisfies a handful of axioms: linearity, dummy, symmetry, ASI, and  proportionality. This argument is a corollary of another claim: any attribution method satisfying ASI and linearity is the difference of two cost-share solutions \citep[Thm 4.1]{sundararajan2020many}. By breaking up an attribution into two cost-share solutions, the aim is to apply cost-share results. The argument roughly is as follows: for any attribution, input, baseline, and function, they use ASI to formulate the attribution as $A(x,0,F)$, with $x>0$. They write $F=F^+-F^-$, where $F^+$ and $F^-$ are non-decreasing. Then by linearity, $A(x,0,F) = A(x, 0, F^+ -F^-) = A(x,0,F^+) - A(x,0,F^-)$, which, the claim states, is the difference of two cost-share solutions. However, there are methods that satisfy ASI and linearity, but generally give negative values for cost-sharing problems. Thus neither $A(x,0,F^+)$ nor $A(x,0,F^-)$ are necessarily cost-share solutions to cost-share problems. See Appendix~\ref{OtherClaimsCE} for a counterexample.

\section{Establishing Uniqueness Claims with Non-Decreasing Positivity}

We now seek to salvage the uniqueness claims identified in the previous section for a robust set of functions. To this end, we introduce the axiom of \textit{non-decreasing positivity}~(NDP). We say that $F$ is non-decreasing from $x'$ to $x$ if $F (\g(t))$ is non-decreasing for every monotone path $\g(t)\in\G(x,x')$ from $x'$ to $x$. We can then define NDP as follows:

\begin{definition}
A BAM $A$ satisfies NDP if $A(x,x',F)\geq0$ whenever $F$ is non-decreasing from $x'$ to $x$.
\end{definition}

$F$ being non-decreasing from $x'$ to $x$ is analogous to a cost function being non-decreasing in the cost-sharing context. NDP is then analogous to requiring cost-shares to be non-negative. Put another way, NDP states that if $F(y)$ does not decrease when any input $y_i$ moves closer to $x_i$ from $x_i'$, then $A(x,x',F)$ should not give negative values to any input. The addition of NDP enables Theorem~\ref{thm1} to extend closer to the DL context.

\begin{theorem}{(Characterization Theorem with NDP)}\label{thm2}
Let $x'$ be fixed. Define $\mathcal{F}^1$ to be the intersection of $\mathcal{F}^D$ and $C^1$. Define $\A^1$ to be the set of baseline attributions with the domain restricted to $\F^1$. Then the following are equivalent:

\begin{enumerate}
    \item $A\in\A^1$ satisfies completeness, linearity, sensitivity(b), and NDP.
    \item $A\in\A^1$ is an ensemble of monotone path methods.
\end{enumerate}
\end{theorem}

A sketch of the proof is as follows. Let $x$ be fixed, and $F \in \F^1$. It can be shown that the behavior of $F$ outside of $[x,x']$ is irrelevant to $A(x,x',F)$. Using this, apply a coordinate transform $T$ that maps $[x,x']$ onto $[|x-x'|,0]$, so that $A(x,x',F) = A^0(0,|x-x'|,F^0)$, where $A^0, F^0$ have proper domains to apply Theorem~\ref{thm1}. $F^0$ is $C^1$ and defined on a compact domain, so its derivative is bounded, and there exists $c\in \R^n$ such that $F^0(y) + c^Ty$ is non-decreasing in $y$. Apply Theorem~\ref{thm1} to $A^0(x,x',F^0(y)+c^Ty)$ and simplify to show $A(x,x',F)$ is an ensemble of path methods for function $F^0$ and paths in $\G[|x-x'|,0]$. Reverse the transform to get the ensemble in terms of $F$ and $\G(x,x')$.

To expand Theorem~\ref{thm2} further to non-$C^1$ functions, we consider a class of feed forward neural networks with layers composed of real-analytic functions and the max function. These include connected layers, activation functions like tanh, mish, swish, residual connections,\footnotemark{} as well as ReLU and Leaky ReLU which can be formulated in terms of a max function. Denote the set of neural networks composed of these layers $\F^2$. We then define $\A^2$ to be any $A\in\A^\text{D}$ with domain limited to $F\in\F^1\cup\F^2$. Thus $\A^2$ is a more robust attribution method than $\A^1$ in that it is defined for a broader class of functions.

\footnotetext{Products, sums, and compositions of analytic functions are analytic. Quotients of analytic functions where the denominator is non-zero are analytic.}

We begin with a lemma regarding the topology of the domain of $F\in\F^2$.

\begin{lemma}\label{lemma2}
Suppose $F\in\F^2$. Then $[x,x']$ can be partitioned into a nonempty region $D$ and it's boundary $\partial D$, where $F$ is real-analytic on $D$, $D$ is open with respect to the topology of the dimension of $[x,x']$, and $\partial D$ is measure $0$.
\end{lemma}

We now present a claim extending Theorem~\ref{thm2} to non-$C^1$ functions. Let $D$ denote the set as described above, and denote the set of points on the path $\g$ by $P^\g$. 

\begin{theorem}{(Extension to class of non-$C^1$ functions)}\label{thm3}
Let $x'$ be fixed. Suppose $A\in\A^2$ satisfies completeness, linearity, sensitivity(b), and NDP, and that $F\in\F^2$. For some $x\in[a,b]$, let $\mu^x$ be the measure on $\G^m(x.x')$ from Theorem~\ref{thm2}. If $A(x,F)$ is defined, and for almost every path $\g\in\G^m(x,x')$ (according to $\mu^x$), $\partial D\cap P^\g$ is a null set, then $A(x,F)$ is equivalent to the usual ensemble of path methods.
\end{theorem}

With the addition of NDP, we also establish the other uniqueness claims of Section~\ref{OtherUniquenessClaims}. For details, see Appendix~\ref{OtherClaimsCE}.

\subsection{Lipschitz Continuity}
DL models can be extremely sensitive to slight changes in the input image~\cite{goodfellow2014explaining}. %
It stands to reason that IG should also have increased sensitive in the output for more sensitive models, and less sensitivity in the output for less sensitive models. The question of whether IG is locally Lipshchitz, and what its local Lipschitz constant is, has been studied previously by experimental means. Previous works searched for the Lipschitz constant when the domain is restricted to some ball around the input, either by Monte Carlo sampling~\citep{yeh2019fidelity} or exhaustive search of nearby input data~\citep{alvarez2018robustness}. In contrast to these, we provide theoretical results on the global sensitivity of IG for two extremes: a model with a discontinuous gradient (as with a neural network with a max or ReLU function), and a model with a well behaved gradient:

\begin{theorem}\label{thm4}
Let $F$ be defined on $[a,b]$, $x'$ be fixed. If $F$ has the usual discontinuities due to ReLU or Max functions, then $\I(x,F)$ may fail to be Lipschitz continuous in $x$. If $\nabla F$ is Lipschitz continuous with constant $L$ and $|\pxi|$ attains maximum $M$, then $\I_i(x,F)$ is Lipschitz continuous in $x$ with Lipschitz constant at most $M + \frac{|a_i-b_i|}{2}L$. 
\end{theorem}

\section{Distribution Baseline Axioms}
As mentioned in \ref{sectionMod&Extend}, some extensions of IG use a distribution of baselines. Here we give a formal definition of the distributional IG, and comment on some axioms it satisfies. We denote the set of distributions on the input space by $\mathcal{D}$. The set of distributional attributions, $\mathcal{E}$, is then defined as the set containing all functions of the form $E:[a,b]\times \mathcal{D}\times \F \rightarrow \R^n$. Given a distribution of baselines images $D\in\mathcal{D}$, we suppose the baseline random variable $X'\sim D$. Then the \textit{distributional IG} is given by $\EG(x,X',F) := \E_{X'\sim D}\I(x,X',F)$. Particular axioms, namely implantation invariance, sensitivity(b), linearity, and ASI can be directly carried over to the baseline attribution context. Distributional IG satisfies these axioms. The axioms of sensitivity(a), completeness, symmetry preserving, and NDP do not have direct analogues. Below we identify distributional attribution axioms that extend sensitivity(a), completeness, and symmetry preserving axioms to the distributional attribution case. Distributional IG satisfies these axioms as well.\footnotemark{} See Appendix~\ref{app:distributionalAxioms} for details.

\footnotetext{Completeness has been observed by \citet{erion2021improving}.}

Let $E\in\mathcal{E}$, $D\in\mathcal{D}$, $X'\sim D$, and $F,G\in\F$:
\begin{enumerate}
    \item Sensitivity(a): Suppose $X'$ varies in exactly one input, $X_i'$, so that $X_j' = x_j$ for all $j\neq i$, and $\E F(X')\neq F(x)$. Then $E_i(x,X',F) \neq 0$.
    \item Completeness: $\sum_{i=1}^{n} E_i (x,X',F) = F(x)-\E F(X')$.
    \item Symmetry Preserving: For a given $i$, $j$, define $x^*$ by swapping the values of $x_i$ and $x_j$. Now suppose that for all $x$, $F(x) = F(x^*)$. Then whenever $X'_i$ and $X'_j$ are exchangeable\footnotemark{}, and $x_i = x_j$, we have $E_i (x,X',F) = E_j (x,X',F)$.
    \item NDP: If $F$ is non-decreasing from every point on the support of $D$ to $x$, then $E(x,X',F)\geq 0$.
\end{enumerate}
\footnotetext{$X_i$ and $X_j$ are exchangeable if $X$ and $X^*$ are identically distributed.}

\section{Internal Neuron Attributions}\label{sec:InternalNeuronAttributions}
\subsection{Previous Methods}\label{sec:InternalNeuronIntro}
Previous works apply IG to internal neuron layers to obtain internal neuron attributions. We review their results before discussing extensions. Suppose $F$ is a single output of a feed forward neural network, with $F:[a,b]\rightarrow \R$. We can separate $F$ at an internal layer such that $F(x) = G(H(x))$. Here $H:[a,b] \rightarrow \R^m$ is the first half of the network outputting the value of an internal layer of neurons, and $G:\R^m \rightarrow \R$ is the second half of the network that would take the internal neuron values as an input. We assume the straight line path $\g$, although other paths can be used. Following~\citet{dhamdhere2018important}, the flow of the gradient in $\I_i$ through neuron $j$, labeled $\I_{i,j}$, is given by:

\small
\begin{equation}\label{IG_ij}
    \I_{i,j} = (x_i-x_i') \int_0^1 \frac{\partial G}{\partial H_j}(H(\g)) \frac{\partial H_j}{\partial x_i} (\g) dt
\end{equation}
\normalsize

By fixing the input and summing the gradient flow through each internal neuron, we get $\I_i$, or, what we equivalently denote for this context, $\I_{i,*}$. This is what we should expect, and is accomplished by moving the sum into the integral and invoking the chain rule.

\small

\begin{equation}\label{IG_i}
    \sum_j\I_{i,j} = (x_i-x_i') \int_0^1 \frac{G \circ H}{d x_i}(\g) dt = \I_{i,*}
\end{equation}
\normalsize

If we fix an internal neuron and calculate the total gradient flow through it for each input, we get an internal neuron attribution, or what \citet{dhamdhere2018important} calls a neuron's~\textit{conductance}:

\small

\begin{align}
    \I_{*,j} &= \sum_i\I_{i,j} \nonumber\\
    &= \sum_i (x_i-x_i') \int_0^1 \frac{\partial G}{\partial H_j}(H(\g)) \frac{\partial H_j}{\partial x_i} (\g) dt \nonumber\\
    &= \int_0^1 \frac{\partial G}{\partial H_j}(H(\g)) \sum_i [ \frac{\partial H_j}{\partial x_i}(\g) \times (x_i-x_i') ] dt \nonumber\\
    &= \int_0^1 \frac{\partial G}{\partial H_j}(H(\g)) \frac{d (H_j\circ\g)}{d t} dt\label{IG_j}
\end{align}
\normalsize

\citet{shrikumar2018computationally} recognized the last line above, which, since $H_j(\g)$ is a path, formulates conductance as a path method. Note that this path may not be monotone, implying the usefulness of non-monotone path methods.

The above formulations can be extended to calculating the gradient flow through a group of neurons in a layer, or through a sequence of neurons in multiple layers. But here we run into a computational issue. Calculating eqs. \ref{IG_ij} or \ref{IG_j} for each neuron in a layer could be expensive using standard programs. We hypothesize this is because they are designed primarily for efficient back-propagation, which finds the gradient of multiple inputs with respect to a single output, not the Jacobean for a large number of outputs.

\subsection{Neuron Attributions for an Input Patch}\label{sec:InternalNeuronPatchInput}
An IG attribution map usually highlights regions or features that contributed to a model's output, e.g., highlighting a face in a picture of a person. 
A pertinent question is: are there internal neurons that are responsible for attributing that feature? In our example, are there neurons causing IG to highlight the face? We propose an answer by attributing to a layer of internal neurons for an input patch.

If we index each input feature, then we can denote a patch of input features by $S$. Then the gradient flow through a neuron $j$ for the patch $S$ is given by:

\small

\begin{equation}\label{IG_S}
\begin{split}
    \I_{S,j} &= \sum_{i\in S} \I_{i,j}\\
    &= \int_0^1 \frac{\partial G}{\partial H_j}(H(\g)) \sum_{i\in S} \frac{\partial H_j}{\partial x_i} (\g) (x_i-x_i') dt
\end{split}
\end{equation}
\normalsize

As noted in section~\ref{sec:InternalNeuronIntro}, computing Eq.~\ref{IG_S} for a full layer of neurons can be expensive. We introduce a speedup inspired by \citet{shrikumar2018computationally}. Let $d$ be a vector with $d_i = x_i-x_i'$ if $i\in S$, and $d_i = 0$ if $i\notin S$. Denote the unit vector $\frac{d}{||d||}$ by $\hat{d}$. Formulating a directional derivative, then taking Reimann sum with $N$ terms, we write:

\small

\begin{equation} \nonumber%
\begin{split}
    \I_{S,j} &= \int_0^1 \frac{\partial G}{\partial H_j}(H(\g)) \sum_{i\in S} \frac{\partial H_j}{\partial x_i} (\g) (x_i-x_i') dt\\
    &= \int_0^1 \frac{\partial G}{\partial H_j} (H(\g)) \hspace{1mm}D_{\hat{d}}H_j (\g)
    \hspace{1mm} ||d|| dt\\
    &\approx ||d|| \int_0^1 \frac{\partial G}{\partial H_j}(H(\g)) \frac{H_j(\g(t)+\frac{\hat{d}}{N})-H_j(\g(t))}{1/N} dt\\
    &\approx ||d|| \sum_{k=1}^N  \frac{\partial G}{\partial H_j}(H(\g(\frac{k}{N}))) \times [H_j(\g(\frac{k}{N})+\frac{\hat{d}}{N})-H_j(\g(\frac{k}{N}))]
\end{split}
\end{equation}
\normalsize

With this speedup, we bypass computing the Jacobean to find $\frac{\partial H_j}{\partial x_i}$ for each input and internal neuron. For an accurate calculation, choose $N$ such that $\I_{S,j}+\I_{S^c,j} \approx \I_{*,j}$.

\section{Experimental Results}
\label{sec:exp}

Here, we present experiments validating the methods in Section~\ref{sec:InternalNeuronAttributions}. We experiment on two models/data sets: ResNet-152~\cite{resnet} trained on ImageNet~\cite{imagenet}, and a custom model trained on Fashion MNIST~\cite{fashionmnist}. Some results for ImageNet appear here, while further results appear in the appendix. The general outline of each experiment is: 1) calculate a performance metric for neurons in an internal layer using IG attributions, 2) rank neurons based on the performance metric, and 3) prune neurons according to rank and observe corresponding changes in model. The goal is to validate the claim that the methods of Section~\ref{sec:InternalNeuronAttributions} identify neurons that contribute to a particular task. The code used in our experiments is available at: \normalsize {\small \url{https://github.com/optimization-for-data-driven-science/XAI}}

\normalsize

\subsection{Preliminaries: Pruning Based on Whole Input Internal Neuron Attributions}

The first experiment (Figure~\ref{fig:ResNetExp1}) calculates a general performance metric for each each internal neuron in a particular layer. We calculate the average neuron conductance over each input in the training set, where the output of $F$ is the confidence in the correct label. We use a black image as a baseline. Following the method of ``deletion and insertion"~\citep{petsiuk2018rise}, we progressively prune (zero out) a portion of the neurons according to their rank and observe changes in model accuracy on the teseting set.\footnotemark{} We zero-out internal neurons because we wish to mask the indication of feature presence, and ResNet uses the ReLU activation function, which encodes no feature presence with a neuron value of zero. These experiments are preformed twice: once on a dense layer ($2^{nd}$ to last), and once on a convolutional layer (the output of the conv2x block).

\begin{figure}[ht]
    \centering
    \includegraphics[width=0.5\linewidth]{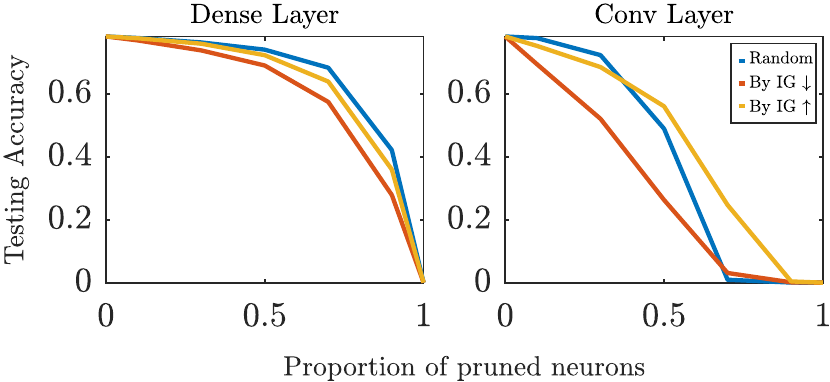}
    \caption{Pruning neurons by conductance values versus random pruning. IG $\downarrow$ means pruning neurons by conductance values in descending order. IG $\uparrow$ means pruning neurons by conductance values in ascending order. }
    \label{fig:ResNetExp1}
\end{figure}

\footnotetext{While similar, our experiment differs from others by zeroing the filter, not ablating it \citep{dhamdhere2018important} or fixing it to a reference input \citep{shrikumar2018computationally}}

When pruning the dense layer, we see that the order of pruning makes little difference in performance. %
We attribute this effect to the dense layer having an evenly distributed neuron importance, something likely in a 1000 category classifier. In the convolutional layer, we see that pruning by descending order rapidly kills the model's accuracy, while pruning by ascending order generally maintains model accuracy better than random pruning. This shows that average conductance can help identify neuron importance.

The second experiment (Figure~\ref{fig:ResNetLemon}) calculates a performance metric indicating a neuron's contribution in identifying a particular image category. We follow~\citet{lundstrom2022explainability} and calculate the same performance metric as previously, but average over a particular category of images (e.g. \textit{Lemons}). We then rank and prune neurons, observing changes in the model's test accuracy identifying the particular category.

In both layers, pruning to kill performance quickly reduces the model's accuracy identifying \textit{Lemon}. This is compared to the median category's performance, and the random pruning baseline. When we prune to keep performance in the dense layer, we see that \textit{Lemon} performs well below the median with random pruning, but swaps to above the median with IG pruning. Pruning in the convoluional layer quickly causes \textit{Lemon} to become very accurate while the median accuracy dips below the random baseline.

\begin{figure}[ht]
    \begin{minipage}{1\linewidth}
        \centering
        \includegraphics[width=0.5\linewidth]{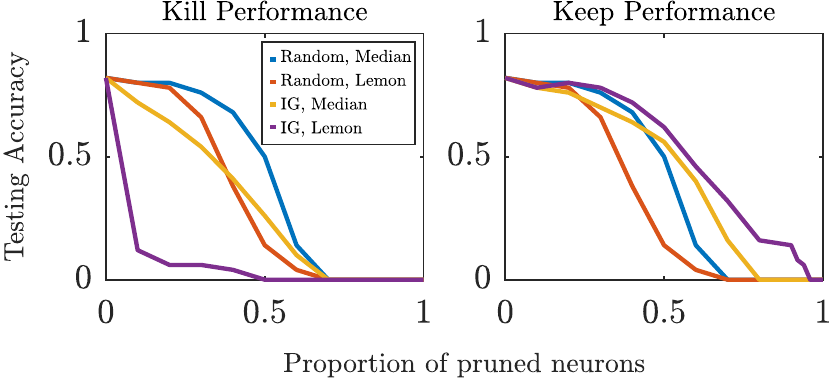}
        \smallskip
    \end{minipage}
    \begin{minipage}{1\linewidth}
        \centering
        \includegraphics[width=0.5\linewidth]{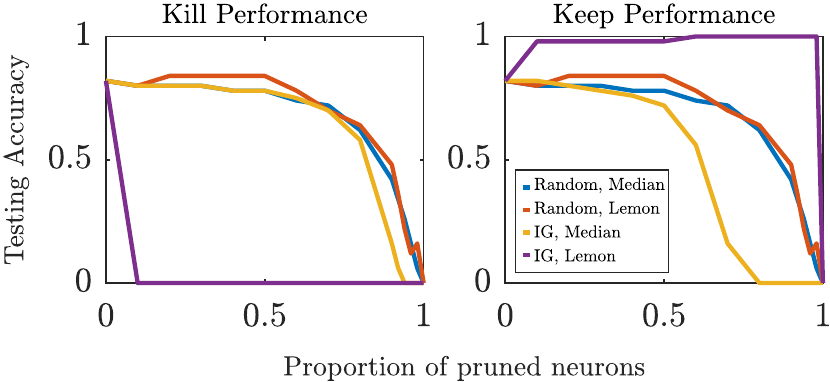}
    \end{minipage}
    \caption{Testing accuracy when neurons are pruned according to their IG values corresponding to the class \textit{Lemon}. Top: neurons pruned in dense layer. Bottom: neurons pruned in a convolutional layer. Left: Neurons pruned by IG values in descending order. Right: Neurons pruned by IG values in ascending order. ``Random, Median'', ``IG, Median'' report median accuracy of all classes for random/ranked pruning. ``Random, Lemon'', ``IG, Lemon'' report accuracy of class \textit{Lemon} for random/ranked pruning.}
    \label{fig:ResNetLemon}
\end{figure}

\subsection{Pruning Based on Internal Neuron Attributions for Image Patchs}
\label{sec:bd-box}

Here we show results of an experiment using image-patch based internal neuron attributions. In a picture of two traffic lights (Figure~\ref{fig:TrafficLightImages}, top-left), we identify an image-patch around one traffic light as a region of interest. We then find the attributions of each internal neuron in a convolutional layer for this image patch and rank them. Using this ranking, we progressively prune the neurons (top-ranked first), periodically reassessing the total IG attributions inside and outside the specified region. This procedure is repeated, instead ranking neurons by their conductance for the image.

\begin{figure}[h]\label{trafficLightExample}
    \begin{center}
        \begin{minipage}{0.2\linewidth}
            \centering
            \includegraphics[width=\linewidth]{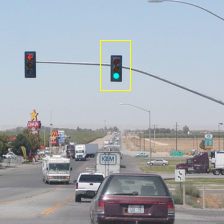}
        \end{minipage}
        \hspace{0.1in}
        \begin{minipage}{0.2\linewidth}
            \centering
            \includegraphics[width=\linewidth]{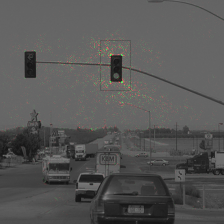}
        \end{minipage}
        \hspace{0.1in}
        \begin{minipage}{0.2\linewidth}
            \centering
            \includegraphics[width=\linewidth]{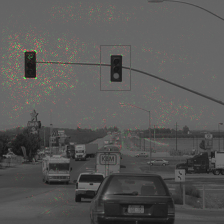}
        \end{minipage}
        \hspace{0.1in}
        \begin{minipage}{0.2\linewidth}
           \centering
           \includegraphics[width=\linewidth]{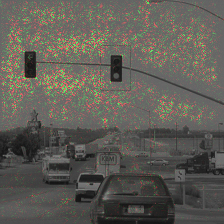}
        \end{minipage}
    \end{center}
    \caption{From left to right are images A,B,C,and D. A: The original image and bounding box indicating specified image patch. B: IG attributes visualized. Green dots show postive IG, red dots show negative IG. We see most IG attributes are within or around the bounding box. C: IG attributes visualized after top $1\%$ of neurons pruned based on image-patch attributions. We see IG attributes moved from the right light to the left light. D: IG attributes visualized after top $1\%$ neurons pruned based on the global ranking. We see IG attributes are scattered.}
    \label{fig:TrafficLightImages}
\end{figure}

From Figure~\ref{fig:TrafficLightExperiment}, we see that using global conductance rankings causes the sum of IG inside and outside the bounding box to briefly fluctuate, then converge to zero. In comparison, pruning by region-targeted rankings consistently causes a positive IG sum outside the box and negative IG sum inside the box. This reinforces the claim that image-patch based rankings give high ranks to neurons causing positive IG values in the bounding box. Interestingly, we also see that ($\sum \I$, all) quickly drops for the global pruning but stays elevated for the regional pruning. By completeness, this indicates the model quickly looses confidence in the former case, but keeps a high confidence for up to 50\% pruning when pruned using region-targeted rankings.

In Figure~\ref{fig:TrafficLightImages}, we prune the top-1\% of neurons in a convolutional layer according to both conductance and image-patch rankings, then re-visualize the IG. The model gives an initial confidence score of 0.9809. When pruning according to conductance, the confidence changes to 0.9391, but the model's attention loses focus, and a broad region receives a cloudy mixture of positive and negative attributions. When pruning according to the image-patch rankings, the confidence score is 0.9958, but the model's attention shifts from the right traffic light to the left one. This validates that the image-patch method indeed highly ranked internal neurons associated with the right traffic light, and ranked neurons is a region-targeted way compared to general neuron conductance. Further experiments can be found in Appendix~\ref{AppendixExperiments}.

\begin{figure}[H]\label{trafficLightNumerical}
    \centering
    \includegraphics[width=0.5\linewidth]{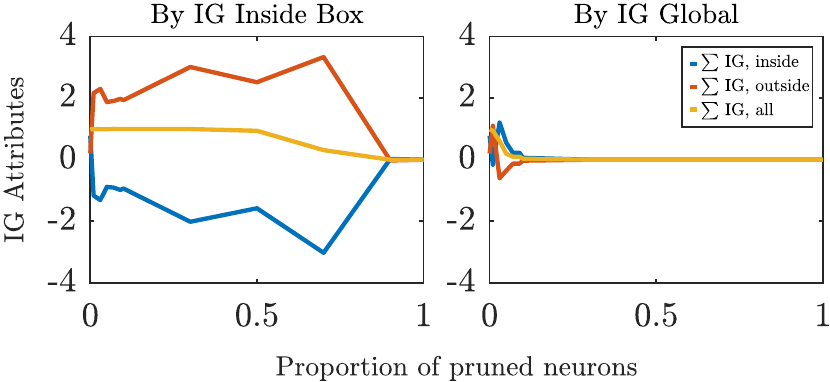}
    \caption{Sum of IG attributes inside and outside the bounding box when neurons are pruned according to certain rankings. Left: Neurons are pruned based on IG global ranking. Right: Neurons are pruned based on the IG ranking inside the bounding box. }
    \label{fig:TrafficLightExperiment}
\end{figure}

\section{Summary}
In this paper, we touched on several aspects of IG. We showed that three uniqueness claims of IG  were false due to discrepancies between the cost-sharing and deep-learning function spaces. With the addition of NDP, these results can be salvaged. We showed that depending on the behaviour of the model, IG may or may not be Lipschitz in the input image. We presented distribution-baseline analogues of certain axioms from \citet{sundararajan2017axiomatic}, all of which the distributional IG satisfy. We reviewed extensions of IG to internal neurons and introduced an efficient means of calculating internal neuron attributions for an image patch. Finally, we presented experiments validating internal neuron attributions using IG.

\bibliography{ref_XAI}
\bibliographystyle{unsrtnat}

\clearpage
\onecolumn

\appendix
\section*{Appendix}

\section{Explanation of Axioms}\label{ExplanationOfAxioms}
Here we give an explanation of each axiom. 1)~Sensitivity(a) stipulates that if altering a baseline by a single input yields a different output, then that input should have some attribution score. 2)~Implementation Invariance states that an attribution method should depend on the form of the function alone, not by any particular way it is coded up. In our formulation, this is a given, but it is possible to consider attribution methods that are a function of model implementation. 3) Completeness states that the sum of the attributions equals the change in function value. It allows an interpretation of each attribution as a contribution to a portion of the function value change. It does this by ensuring that the attribution function has a complete accounting of said change. 4)~Linearity would be desirable in ensemble voting models, and indicates that the attribution to an input is the weighted sum of its attributions for the individual models, with the weights equal to the ensemble weights. 5)~Sensitivity(b) is called dummy or dummy-consistency in \citet{friedman2004paths}, and simply means that if an input does not affect the output, then it should have zero attribution. 6)~Symmetry Preservation indicates that if two variables are universally interchangeable in the function, and their values are identical in the input and baseline, then their attributions should not differ. 7) Affine Scale Invariance implies that if the baseline and input are shifted are stretched, but the model was adjusted for this shifting and stretching, then the attributions would not change. As an example, if a model were designed for Fahrenheit, and then adjusted to Celcius, the attributions would not change for the same absolute-temperature input and baseline. For further comments, see~\citet{sundararajan2017axiomatic}, \citet{sundararajan2020many}.

\section{Figure 1 from ~\citet{sundararajan2017axiomatic}}\label{AAfig1}
\begin{figure}[H]
\centering
\includegraphics[width=6cm]{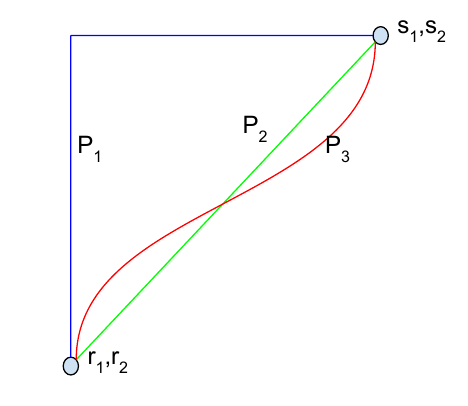}\\
\caption{Three paths between an a baseline $(r_1,r_2)$ and an input $(s_1,s_2)$. Each path corresponds to a different attribution method. The path $P_2$ corresponds to the path used by integrated gradients.}
\label{fig:classExamples}
\end{figure}

\section{Counterexample to Claim~\ref{Claim:FalsePathThm}}\label{claim1ce}
Let $F(x_1,x_2) = x_1 x_2$ be defined on $[0,1]^2$, $x' = (0,0)$, $x = (1,0)$. Suppose that $A^\g$ is defined by a monotone path. Note that $\frac{\partial F}{\partial x_i} \geq 0$, $\frac{\partial \g_i}{\partial t} \geq 0$ for all $i$. Thus $A^\g(x,x',F) \geq 0$ by Eq.~\ref{eq:PathMethod}. Thus any ensemble of monotone path methods has a non-negative output for the input $(x,x',F)$.

Let $\g'$ be the path that travels via a straight line from $(0,0)$ to $(0,1)$, then to $(1,1)$, and ends at $(1,0)$. $A^{\g'}$ satisfies completeness, linearity, sensitivity(b), and implementation invariance. $A^{\g'}(x,x',F) = (1,-1) \ngeq 0$. Thus not every baseline attribution that satisfies completeness, linearity, sensitivity(b), and implementation invariance is a probabilistic ensemble of monotone path methods.

\section{Proof of Lemma~\ref{lemma1}} \label{Proof:lemma1}

\begin{proof}
Suppose $A$ satisfies linearity and sensitivity(b), and $F\in\F$ has defined $\nabla F$ in $[a,b]$. Let and $x, x'\in [a,b]$, and let $G\in\F$ be such that $\frac{\partial F}{\partial x_i} = \frac{\partial G}{\partial x_i}$ for some $i$. Then $\frac{\partial (F-G)}{\partial x_i} = 0$, and by sensitivity(b), $A_i(x,x',F) - A_i(x,x',G) = A_i(x,x',F-G) = 0$. Thus $A_i(x,x',F) = A_i(x,x',G)$, and $A_i$ is a function solely of $x, x'$ and $\frac{\partial F}{\partial x_i}$. By extension, $A(x,x',F)$ is a function solely of x,x', and $\nabla F$.

\end{proof}

\section{Counterexamples to Other Uniqueness Claims and Proof with NDP}\label{OtherClaimsCE}

\subsection{Counterexample to \citet[Proposition 1]{xu2020attribution} and Proof with NDP}

In this section we present a counterexample to \citet[Proposition 1]{xu2020attribution} and establish the claim with the addition of NDP. The original statement of \citet[Proposition 1]{xu2020attribution} is as follows:
\begingroup
\addtolength\leftmargini{-0.18in}
\begin{quote}
\it ``Path methods are the only attribution methods that always satisfy Dummy, Linearity, Affine Scale Invariance and Completeness."
\end{quote}
\endgroup

As with the statement of \citet[Proposition 2]{sundararajan2017axiomatic}, the definition of ``path methods" here is informed by the work of \citet{friedman2004paths}, which is a referent given as proof of the theorem and specifies the statement. We first rigorously re-state the claim, filling in gaps. The work of \citet[Theorem 1]{friedman2004paths}, which was given above as Theorem~\ref{thm1}, is given in the context of monotone path methods and their ensembles. Since this theorem is referenced without further proof, it is assumed they do not mean to include non-monotone path methods, since otherwise the theorem would not be justifiably applied. Furthermore, it is known that if a path method satisfies the axioms, then an ensemble of path methods satisfies the axioms. Thus we assume they mean to include ensembles of path methods. Thus, we can interprete the statement as:
\begin{claim}{(\citet[Proposition 1]{xu2020attribution})}\label{claim_xu_prop}
If an attribution method satisfies dummy, linearity, ASI, and completeness, then that method is an ensemble of monotone path methods.
\end{claim}

Here we present a counterexample to this claim in the form of a non-monotone path method that satisfies the axioms. Let $n=2$, $[a,b] = [0,1]$. Define $\g(x,x',t)$ as follows. Set $T$ to be the affine transformation $T(y) = x' + (x-x')\odot y$. Inspired by \citet{friedman2004paths}'s treatment of ASI, we define a non-monotone path method $\g(x,x',t)$ as such. We set $\g(1,0,t)$ as the constant velocity path which travels in straight lines as such: $(0,0)\rightarrow(1,0)\rightarrow(1,1)\rightarrow(0,1)\rightarrow(0,0)\rightarrow(1,1)$. Define $\g(x,x',t) = T(\g(0,1,t))$. Thus $\g(x,x',t)$ is affine transformation of the reference path $\g(0,1,t)$. Note that $\g(0,1,t)\in[0,1]$, so if $x,x'\in[0,1]$, then the path $\g(x,x',t)\in[x,x']$, and will not exit out of the box $[0,1]$. This ensures that $A^\g(x,x',F)$ is well defined for any $x,x'$.

Note that $A^\g$ is a non-monotone path method satisfying completeness, dummy, and linearity. To complete the counterexample, it remains to show that $A^\g$ satisfies ASI. Let $T'$ be any affine transformation as in the definition of ASI. All that remains is to show that if $x,x',T'(x),T'(x')\in [a,b]$, then $A^\g(x,x',F)=A^\g(T'(x),T'(x'),F \circ T'^{-1})$. First we note that $T'(T(\g(1,0,t))) = \g(T'(T(1)),T'(T(0)),t)= \g(T'(x),T'(x'),t)$. Applying this, we show that for any index $i$:

\begin{equation}
\begin{split}
    A_i^\g(x,x',F) &= \int_0^1\pxi(\g(t))\frac{d \g_i}{d t} dt\\
    &= \int_0^1 \pxi(T(\g(1,0,t)))\frac{d (T(\g(1,0,t)))_i}{d t} dt\\
    &= \int_0^1 \pxi(T'^{-1}(T'(T(\g(1,0,t)))))\frac{d (T'^{-1}(T'(T(\g(1,0,t)))))_i}{d t} dt\\
    &= \int_0^1 \pxi(T'^{-1}(\g(T'(x),T'(x'),t)))\frac{d (T'^{-1}(\g(T'(x),T'(x'),t)))_i}{d t} dt\\
    &= \int_0^1 \pxi(T'^{-1}(\g(T'(x),T'(x'),t)))\frac{d T'^{-1}_i}{d x_i}(\g(T'(x),T'(x'),t)))\frac{d (\g(T'(x),T'(x'),t))_i}{d t} dt\\
    &= \int_0^1 \frac{\partial (F \circ T'^{-1})}{\partial x_i}(\g(T'(x),T'(x'),t))\frac{d (\g(T'(x),T'(x'),t))_i}{d t} dt\\
    &= A_i^\g(T'(x),T'(x'),F \circ T'^{-1})
\end{split}
\end{equation}

Thus $A^\g$ is an attribution method that satisfies dummy, linearity, ASI, and completeness, but is not in the form of an ensemble of monotone path methods. We now prove that $A^\g$ is not equivalent to a ensemble of monotone path methods. We do this by introducing a context where $A^\g$ gives negative attributions, and note that monotone path methods (and thus ensembles of monotone path methods) cannot give negative attributions in this context.

Let $F(x_1,x_2) = x_1x_2^2$. We calculate $A^\g(1,0,F)$ by calculating the five straight paths that comprise it. We denote the path $P^1$ to be the path from $(0,0)\rightarrow (1,0)$, $P^2$ to be the path from $(1,0)\rightarrow (1,1)$, and so on. By this decomposition, we have $A^\g(1,0,F) = \sum_{i=1}^5 \I(P^i,F)$, where $P^i$ indicates the input and baseline of IG in the obvious way. Calculating $\I(P^1,F)$,..., $A(P^4,F)$ is simple if we observe that $\pxi=0$ for one variable, which causes that component to be zero, and then apply completeness. This yields: $\I(P^1,F) = (0,0)$, $\I(P^2,F) = (0,1)$, $\I(P^3,F) = (-1,0)$, and $\I(P^4,F) = (0,0)$.

Because parameterization will not affect the integral, we parameterize $P^5(t) = (t,t)$. Calculating $A^{P^5}(1,0,F)$,

\begin{align*}
    \I_1(P^5,F) &= \int_0^1 \frac{\partial F}{\partial x_1} \frac{d P^5_1}{d t} dt\\
    &= \int_0^1 t^2 dt\\
    &= \frac{1}{3}
\end{align*}

By completeness we get $\I(P^5,F) = (\frac{1}{3},\frac{2}{3})$. Thus, $A^\g(1,0,F) = (-\frac{2}{3},\frac{5}{3})$, and $A^\g$ can give negative values to a non-decreasing $C^1$ function with baseline 0 and input 1. Because monotone path methods cannot give negative values in this case, and by extension, ensemble of monotone path methods cannot either, $A^\g$ cannot be represented as an ensemble of monotone path methods.

We note that many non-monotone, piece-wise smooth paths could suffice for the counterexample. We also note that since non-monotone path methods satisfy the above axioms, it is an open questions whether other methods that are not an ensemble of path methods also satisfy the axioms.

We now establish Claim~\ref{claim_xu_prop} with the additional assumption of NDP for a particular class of functions.

\begin{corollary}{(Claim~\ref{claim_xu_prop} with NDP for $\F^1\cup\F^2$ Functions)}
Let $x'$ be fixed, $x\in[a,b]$. Suppose $A\in\A^2$ satisfies dummy, linearity, completeness, ASI, and NDP. 1) If $F\in\F^1$, $A(x,x',F)$ is equivalent to the usual ensemble of path methods. 2) If  $F\in\F^2$,  let $\mu^x$ be the measure on $\G^m(x.x')$ from Theorem~\ref{thm2}. If $A(x,F)$ is defined, and for almost every path $\g\in\G^m(x,x')$ (according to $\mu^x$), $\partial D\cap P^\g$ is a null set, then $A(x,F)$ is equivalent to the usual ensemble of path methods.
\end{corollary}
\begin{proof}
These are two specific cases of Theorems~\ref{thm2} and \ref{thm3}.
\end{proof}

\subsection{Counterexample to \citet[Thm 4.1]{sundararajan2020many} and Proof with NDP}

The original statement of \citet[Thm 4.1]{sundararajan2020many} is as follows:
\begingroup
\addtolength\leftmargini{-0.18in}
\begin{quote}
``\it (Reducing Model Explanation to Cost-Sharing). Suppose there is an attribution method that satisfies Linearity and ASI. Then for every attribution problem with explicand $x$, baseline $x'$ and function $f$ (satisfying the minor technical condition that the derivatives are bounded), then there exist two costsharing problems such that the resulting attributions for the attribution problem are the difference between cost-shares for the cost-sharing problems."
\end{quote}
\endgroup

 \citet{sundararajan2020many} defines ``cost-sharing problems" to be attributions where $x'=0, x\geq 0$, and $F$ is non-decreasing in each component. Interprating ''cost-shares", we look to the referenced work, \citet{friedman1999three}, which restricts cost-share solutions to non-negative solutions to cost-sharing problems. A restatement of the theorem is then:

\begin{claim}{(\citet[Proposition 1]{sundararajan2020many})}\label{claim_sundararajan_prop_2}
Suppose $A$ is an attribution method that satisfies linearity and ASI. Then for every attribution problem $x,x'$, and $F$ with bounded first derivative:
\begin{enumerate}
    \item There exists $z,\bar{z}\geq 0$, $G,H$ non-decreasing \hfill (There are 2 cost-share problems)
    \item $A(x,x',F) = A(z,0,G) - A(
    \bar{z},0,H)$ \hfill (The original attribution equals the difference\newline \null\hfill between attributions for the cost-share problems)
    \item  $A(z,0,G)$, $A(\bar{z},0,H)\geq 0$ \hfill ($A$ gives cost-share solutions to the cost-share problems)

\end{enumerate}
\end{claim}

Let $n=1$, and define the attribution method $A$ by $A(x,x',F) := F(x')-F(x)$. Note that $A$ satisfies linearity since:

\begin{equation}
\begin{split}
    A(x,x',F_1+F_2) &= F_1(x')+F_2(x')-F_1(x)-F_2(x) = A(x,x',F_1)+A(x,x',F_2)
\end{split}
\end{equation}

$A$ also satisfies ASI since for any linear transformation $T$ we have:

\begin{equation}
\begin{split}
    A(T(x),T(x'),F \circ T^{-1}) &= F \circ T^{-1}(T(x')) - F \circ T^{-1}(T(x))\\
    &= F(x')-F(x) \\
    &= A(x,x',F)
\end{split}
\end{equation}
Now let $x=1$, $x' = 0$, $F(y) := y$. We proceed by contradiction. Suppose there exists $z,\bar{z}\geq 0$, $G,H$ non-decreasing such that $A(z,0,G)$, $A(\bar{z},0,H)\geq 0$ and $A(1,0,F) = A(z,0,G) - A(\bar{z},0,H)$. Now observe that $A(1,0,F) = F(0) - F(1) = -1$, which implies $A(\bar{z},0,H)>0$. However, $A(\bar{z},0,H) = H(0) - H(\bar{z})\leq 0$, a contradiction. Thus the theorem does not hold for $A$ with the stipulated $x,x',F$, and is false.

We now establish Claim~\ref{claim_sundararajan_prop_2} with the addition of NDP.

\begin{theorem}{(Claim~\ref{claim_sundararajan_prop_2} with NDP)}
Suppose $A$ is an attribution method that satisfies linearity, ASI, and NDP. Then for every $x,x'$, and $F$ with bounded first derivative:
\begin{enumerate}
    \item There exists $z,\bar{z}\geq 0$, $G,H$ non-decreasing \hfill (There are 2 cost-share problems)
    \item $A(x,x',F) = A(z,0,G) - A(
    \bar{z},0,H)$ \hfill (The original attribution equals the difference\newline \null\hfill between attributions for the cost-share problems)
    \item  $A(z,0,G)$, $A(\bar{z},0,H)\geq 0$ \hfill ($A$ gives cost-share solutions to the cost-share problems)

\end{enumerate}
\end{theorem}
\begin{proof}
We follow the proof from \citet[Proposition 1]{sundararajan2020many}, but employ NDP. Since $A$ satisfies ASI, there exists an affine transformation $T$ such that $A(x,x',F) = A(T(x), 0,F \circ T^{-1})$, where $T(x)\geq 0$ and $F \circ T^{-1}$ has bounded derivative. Since $F \circ T^{-1}$ has a bounded derivative, there exists a $c\in\R^n$ such that $F \circ T^{-1}+c^Ty$, $c^Ty$ are non-decreasing. By linearity, $A(T(x),0,F \circ T^{-1}) = A(T(x),0,F \circ T^{-1}+c^Ty) - A(T(x),0,c^Ty)$. Because $A$ satisfies NDP, we have $A(T(x),0,F \circ T^{-1}+c^Ty)$, $A(T(x),0,c^Ty)\geq 0$.
\end{proof}

\section{Comment on Conjecture 1}\label{conj1comm}

If no qualifications are put on the set of paths that $\mu^x$ is supported on, then $A$ may take on infinite values, contradicting completeness, or may simply be undefined. Consider the following example. Let $n=2$, $[a,b] = [0,1]$. Let $F(y) = y_1y_2$, $x = (1,1)$, $x' = (0,0)$. Define the path $\g^n(x,x',t)$ to be the path obtained by traveling completely around the boundary of the domain clockwise $n$ times, then following the straight line from $(0,0)$ to $(1,1)$. We define $\g^{-n}(x,x',t)$ similarly to $\g^n$, but with counterclockwise paths. $A^{\g^0}(x,x',F) = (0.5,0.5)$. $A^{\g^n}(x,x',F) = (0.5+n,0.5-n)$, $n\in \Z$. Now define the support of $\mu^x(\g)$ to be $\{\g^{(-2)^k}: k \in \N\}$. We then define $\mu^x$ on it's support to be $\mu^x(\g^{(-2)^k}) = \frac{1}{2^{k}}$.

\begin{align*}
    & A(x,x',F)\\
    =& \int_{\g \in \G(x,x')} \A^\g(x,x',F) d\mu^x(\g)\\
    =& \sum_{k=1}^\infty (0.5+(-2)^k,0.5-(-2)^k) \frac{1}{2^k}\\
    =& \sum_{k=1}^\infty (\frac{0.5}{2^k} + (-1)^k, \frac{0.5}{2^k} - (-1)^k)
\end{align*}

The above sum is not convergent in either component, so $A(x,x',F)$ is not defined.

A similar construction only allowing clockwise paths may yield $A(x,x',F) = (\infty, -\infty)$, contradicting completeness.

\section{Proof of Theorem~\ref{thm2}}\label{Proof:Thm2}
\begin{proof}

We begin by supposing the assumptions. Let $x'$ be fixed, $\F^1$ and $\A^1$ be as stipulated, and $A\in\A^1$. We introduce the notation $\A^1(c,d)$, $c,d\in \R^n$, to be defined as the set $\A^1$, but with specified region $[c,d]$ instead of $[a,b]$. The set $\F^1(c,d)$ is defined likewise.

$2) \rightarrow 1)$: Suppose $A$ is an ensemble of monotone path methods as in the theorem statement. It is trivial to show that $A$ satisfies linearity, completeness, and sensitivity(b). Suppose $F$ is non-decreasing from $x'$ to some $x$. Then for any monotone path $\g$ from $x'$ to $x$, $A^\g(x,x',F)\geq 0$. Thus $A(x,x',F)\geq 0$, and $A$ satisfies NDP.

$1) \rightarrow 2)$: Let $A$ satisfy completeness, linearity, sensitivity(b), and NDP. Let $F\in\F^1(a,b)$ and $x\in[a,b]$. WLOG, we may assume that $F(x') = 0$, since if not, consider $G(y):=F(y)-F(x')$ and apply Lemma~\ref{lemma1}.

Our strategy will be to first define a transform such that $A$ can be represented as a baseline attribution with baseline $0$. Define $T:\R^n\rightarrow\R^n$ as $T_i(y) = (y_i-x'_i)\times(-1)^{\mathbbm{1}_{x_i'>x_i}}$. One can think of $T$ as a transform from the baseline $x'$ space to the baseline $0$ space. $T$ transforms $[a,b]$, by shifting and reflections about axes, into some other rectangular prism $[c,d]$, for some $c,d\in\R^n$. More importantly, $T$ transforms $x'$ to $0$ and $x$ to $|x-x'|$. Specifically, we get $T([x,x']) = [|x-x'|,0]$, with $T(x') = 0$ and $T(x) = |x-x'|$. Note further that $T$ transforms the set of monotone paths from $x'$ to $x$ into the set of monotone paths from $0$ to $|x-x'|$, or $T(\G^m(x,x')) = \G^m(|x-x'|,0)$. $T$ is one-to-one and has a well defined inverse over $\R^n$. So one can think of $T^{-1}$ as a transform from the baseline 0 space to the baseline $x'$ space.

For $y,y'\in [c,d]$, $G\in \F^1(c,d)$, define $A' \in \A^1(c,d)$ by $A'(y,y',G) := A(T^{-1}(y),T^{-1}(y'), G  \circ  T)$. Essentially $A'$ is a reformulation of $A$ in the baseline 0 space. By definition, $A(x,x',F) =  A'(|x-x'|,0,F  \circ  T^{-1})$. $A'$ satisfies completeness, linearity, sensitivity(b), and NDP.

Note that to apply Theorem~\ref{thm1}, we must restrict the domain of $A'$ to not include inputs with negative components. If we restrict the domain of $A'$, it is not clear that this attribution will behave the same. It seems possible that the attribution $A'(x,x',F)$ depends on the behavior of $F$ in the domain we want to remove. If this were the case, issues could arise, such as the restricted $A'$ not being equivalent to the unrestricted $A'$. To address this issue, we turn to the development of an important lemma.

\begin{lemma}\label{lemma3}
If $A\in\A^1$ satisfies completeness, linearity, sensitivity(b), and NDP, then $A(x,x',F)$ is determined by $x,x'$ and the behavior of $F$ inside $[x,x']$.
\begin{proof}
Suppose $G,H\in\F^1$ have the same behavior in $[a,b]$. So for $y\in[a,b]$, $G(y)-H(y) = 0 = H(y)-G(y)$. Thus both are non-decreasing from $x'$ to $x$. Because $A$ satisfies NDP, $A(x,x',H-G)\geq0$, and $A(x,x',G-H) = - A(x,x',H-G)\geq0$. Thus $0 = A(x,x',G) - A(x,x',H)$, and $A(x,x',G) = A(x,x',H)$.
\end{proof}

\end{lemma}

Now we define a BAM to apply Theorem~\ref{thm1} on. Define $A'':[T(x),0]\times\F^0(T(x),0)\rightarrow \R$ as such: $A''(y,G) := A'(y,0,H)$, where $H\in\F^1[c,d]$ is any function such that $H=G$ when restricted to $[T(x),0]$. $A''$ is a properly defined BAM by Lemma~\ref{lemma3}. Note that for $G\in \F^1$ with $G(0) = 0$, $G$ non decreasing, and $y\in[T(x),0]$, we may go backwards and say $A'(y,0,G) = A''(y,G)$. Furthermore, $A''$ satisfies completeness, linearity, sensitivity(b), and NDP.

Write $F^0 = F \circ T^{-1}$. $F^0$ is a $C^1$ function defined on a compact domain, so $\nabla F^0$ is bounded. So there exists $c\in\R^n$ such that $\nabla (F^0(y) + c^Ty ) = \nabla F^0 + c\geq 0$ on the compact domain. This implies that $F^0(y) + c^Ty  $ is non-decreasing, $C^1$, with $F^0(0) = 0$. So $F^0(y) + c^Ty \in\A^0$. Employing Theorem~\ref{thm1}, there exists a measure $\mu$ such that:

\begin{align*}
    &A_i(x,x',F(y) + c^TT(y) ) \\
    =& A_i'(T(x),0,F^0(y) + c^Ty)\\
    =& A_i''(T(x),F^0(y) + c^Ty)\\
    =& \int_{\g\in \G^m(T(x),0)} A_i^{\g}(T(x),0,F^0(y) + c^Ty ) \times d\mu(\g)
\end{align*}

Inspecting the interior term, we find that for $\g$ a monotone path from $0$ to $T(x)$,

\begin{align*}
    \hspace{3.05cm}&A_i^{\g}(T(x),0,F^0(y) + c^Ty )\\
    =& \int_0^1 [\frac{\partial F^0}{\partial \g_i} + c_i] \frac{d \g_i}{d t} dt\\
    =& \int_0^1 \frac{\partial F  \circ  T^{-1}}{\partial \g_i} \times \frac{d \g_i}{d t} dt + \int_0^1 c_i \frac{d \g_i}{d t} dt \\
    =& \int_0^1 \frac{\partial F}{\partial (T^{-1} \circ \g)_i}(T^{-1}(\g(t)) \times \frac{\partial T^{-1}_i}{\partial \g_i}(\g(t)) \times \frac{d \g_i}{d t} dt + c_i(T_i(x)-T_i(x'))\\
    =& \int_0^1 \frac{\partial F}{\partial (T^{-1} \circ \g)_i}(T^{-1}(\g(t))) \times \frac{\partial (T^{-1}  \circ  \g)_i}{\partial t} dt + c_iT_i(x)\\
    =& A_i^{(T^{-1}  \circ  \g)}(x,x',F) + c_iT_i(x)
\end{align*}

Set $\mu'(\g) := \mu(T(\g))$ so that $\mu'$ is a measure on the monotone paths from $x'$ to $x$. Combining previous results, we have,
\begin{align*}
    &A_i(x,x',F(y)) + A_i(x,x',c^TT(y) )\\
    =&A_i(x,x',F(y) + c^TT(y) )\\
    =& \int_{\g\in \G^m(T(x),0)} A_i^{\g}(T(x),0,F^0(y) + c^Ty ) \times d\mu(\g)\\
    =& \int_{\g\in \G^m(T(x),0)} [A_i^{(T^{-1}  \circ  \g)}(x,x',F) + c_iT_i(x)] \times d\mu(\g)\\
    =& \int_{\g \in \G^m(T(x),0)} A_i^{(T^{-1}  \circ  \g)}(x,x',F) \times d\mu(\g) + c_iT_i(x) \\
    =& \int_{\g \in \G^m(x,x')} A_i^{\g}(x,x',F) \times d\mu(T(\g))+ c_iT_i(x)\\
    =& \int_{\g \in \G^m(x,x')} A_i^{\g}(x,x',F) \times d\mu'(\g)+ c_iT_i(x)
\end{align*}

From a previous result, $A_i(x,x',G)$ is a function only of $x,x'$ and $\frac{\partial G}{\partial x_i}$. So $A_i(x,x',c^TT(y)) = A_i(x,x',c_iT_i(y))$. By sensitivity(b), $A_j(x,x',c_iT_i(y)) = 0$ for $j\neq i$. So by completeness, $A_i(x,x',c_iT_i(y)) = c_iT_i(x) - c_iT_i(x') = c_iT_i(x)$. Subtracting the term from both sides of the above equation yields:

$$A_i(x,x',F(y)) = \int_{\g \in \G^m(x,x')} A_i^{\g}(x,x',F) \times d\mu'(\g)$$

Note that $A''$ is determined by $A$ and choice of $x$ and $x'$, since $T$ is determined by $x,x'$. Note further that $A''$ and $T$ determines $\mu'$. So for any $F\in\F^1$ and fixed $x'$, we can index on $x$ to get $\mu'^x$, a probability measure on $\G(x,x')$. Thus for a fixed $x'$, $F\in\F^1$, we have:

\begin{align*}
    &A_i(x,x',F)= \int_{\g \in \G^m(x,x')} A_i^{\g}(x,x',F) d\mu'^x(\g)
\end{align*}

\end{proof}

\section{Proof of Lemma~\ref{lemma2}}\label{Proof:lemma2}
\begin{proof}
Proceed by induction. Let $F:\R^n\rightarrow \R$ be a one-layer feed forward neural network with $F\in\F^2$. If the layer is an analytic function, then $F$ is analytic since the composition of analytic functions is analytic. Precisely, $F$ is analytic in the interior of $[x,x']$, which has a boundary of measure $0$. If the layer is a max function, then $F$ is analytic except on the boundary of $[x,x']$ and potentially some hyper-plane, which is a null set. In either case, the result is obtained.

Now suppose that $F:\R^n\rightarrow \R^m$ is a $k$-layer feed forward neural network with $F_i\in\F^2$ for each $i$. Further suppose $[x,x']$ can be partitioned into and open set $D_i$ and $\partial D_i$, where each output $F_i$ is analytic on $D_i$ and $\partial D_i$ a null set. If $H:\R^m\rightarrow \R$ is an analytic function, then $H \circ F$ is analytic on $D = \cap_{i=1}^n D_i$, $\partial D = \cup_{i=1}^n \partial D_i$ is a null set, and $D\cup \partial D$ is a null set.

Now suppose instead that $H$ is a max function. Since the max of more than two functions is the composition of the two-input max function, we will only consider the two-input max. Let $H$ be the max function of the $i^{th}$ and $j^{th}$ components, so that $H \circ F = \max(F_i,F_j)$. First, we inspect points in $D_i \cap D_j$. Let $y\in D_i \cap D_j$, and consider three disjoint cases: 1) if $F_i(y) - F_j(y) \neq 0$, then by continuity of $F$, $F_i - F_j \neq 0$ in some ball centered around $y$. This implies that either $H \circ F \equiv F_i$ or $H \circ F \equiv F_j$ in some ball around $y$. Thus $H \circ F$ is analytic at $y$. Note the set of case 1 points form an open set. Denote the set of case one points $D^1$. 2) If $F_i(y) = F_j(y)$, and $F_i = F_j$ for some ball centered around $y$, then $H = F_i = F_j$ in that ball, and $H \circ F$ is analytic at $y$. Note the set of case 2 points form an open set. Denote the set of case 2 points $D^2.$ 3) We denote the set of all other points in $D_i \cap D_j$ by $D^3$. For $y\in D^3$, $F_i(y) = F_j(y)$, but $F_i \neq F_j$ for some point in every open ball centered at $y$. Set $D = D^1 \cup D^2$, and note that $D$ is open, $H \circ F$ is analytic in $D$.

Since $D_i \cap D_j$ is an open set, there exists a countable sequence of open balls, $\{B_i\}$, such that $D_i \cap D_j = \cup_{i=1}^{\infty} B_i$. For any $B_i$, set $B_i^1:= B_i\cap D^1$, $B_i^2:= B_i\cap D^2$, and $B_i^3:= B_i\cap D^3$. Let $G(y) = F_i(y) - F_j(y)$, and note that since $F_i$, $F_j$ are analytic on $B_i$, $G$ is analytic on $B_i$ also. Note that for any $y\in B_i^3$, $G(y) = 0$. If $m(B_i^3) > 0$, then, $m(\{y\in B_i| G(y) = 0\})>0$, and since $G$ is analytic, $G\equiv 0$ on $B_i$. This is a contradiction, since this implies $B_i = B_i^2$ and $m(B_i^3) = 0$. Thus $m(B_i^3) = 0$, and $D^3 = \cup B_i^3$ is a null set.

$D^1, D^2, D^3$, and $\partial D_i\cup \partial D_j$ partition $[x,x']$, a closed and bounded set. $D= D^1\cup D^2$ is an open set in the interior of $[x,x']$ and $\partial D_i\cup\partial D_j\cup D^3$ is null and thus has no interior. Thus $\partial D = \partial D_i\cup\partial D_j\cup D^3$, a null set.

Since $D^3$ points are boundary points of $D^1$, we have that $\partial D_i$, $\partial D_j$, and $D^3$ are all boundary points of $D$. Since $D^1$, $D^2$, $D^3$, and $\partial D_i \cup \partial D_j$ partition $[x,x']$, and $D$ is an open set in the interior of $[x,x']$, we have $\partial D = D^3 \cup \partial D_i \cup \partial D_j$, a null set.
\end{proof}

\section{Proof of Theorem~\ref{thm3}}\label{Proof:Thm3}
\begin{proof}
Suppose the suppositions of the theorem. We may assume that $x'=0$, $x\geq 0$, for otherwise we may use the transformation technique applied in theorem $2$. Further suppose $x\neq x'$, for otherwise the result is trivial. Denote the open region where $F$ is $C^1$ by $D$. Note that since each layer of $F$ is a Lipschitz function, $F$ is Lipschitz.

We now turn to a useful lemma, but before we do, we give the following definitions. For a given $i$, $y\in[x,x']$, we define a function that travels from one side of the rectangle $[x,x']$, through $y$, and to the other side, while varying only in the $i^{th}$ component. Formally, define $\l^{(y,i)}(t)$ with $0\leq t\leq |x_i-x_i'|$ as such: $\l^{(y,i)}_i(t) = x'_i + \text{sign} (x_i-x_i')  t$, and $\l^{(y,i)}_j(t) = y_j$ for $j \neq i$. We say that $F$ is non-decreasing from $x'$ to $x$ in it's $i^{th}$ component if, for all $y\in [x,x']$, $F \circ \l^{(y,i)}$ is non-decreasing in $t$.

\begin{lemma}\label{lemma4}
Let $A$ satisfy linearity, completeness, sensitivity(b), and NDP. Suppose $F\in\F$ is Lipschitz continuous and non-decreasing from $x'$ to $x$ in its $i^{th}$ component. Then $A_i(x,x',F)\geq 0$.
\end{lemma}
\begin{proof}
Since $F$ is Lipschitz, there exists $c$ with $c_i = 0$ such that for each $j$, $F +c^Ty$ is non-decreasing from $x'$ to $x$ in the $j^{th}$ component. Set $G(y) = F(y)+c^Ty$. For any monotone path $\g$ from $x'$ to $x$, if $t\leq t'$ then $G \circ \g (t) \leq G \circ \g (t')$, implying $G$ is non-decreasing from $x'$ to $x$. Note that $\partial_i (F-G) = -\partial_i (c^Ty) = 0$. Thus, by Dummy, $A_i(x,x',F) = A_i(x,x',F - G) + A_i(x,x',G) = A_i(x,x',G) \geq 0$.
\end{proof}

Our goal now is to construct a sequence of $C^1$ functions $\{F_m\}$ such that $\lim_{m\rightarrow \infty} A_i(x,x',F_m) = A_i(x,x',F)$. Fix $i$ and define $f = \pxi$ for $x\in D$. For $x\in\partial D$, define $f(x) = -L$, where $L$ is the Lipschitz constant of $F$. $f$ is continuous in $D$ and minimized on $\partial D$, thus $f$ is lower semi-continuous. By Baire's Theorem, there exists a monotone increasing sequence of continuous functions, $\{g_m\}$, such that $g_m\rightarrow f$ point-wise. Because $f$ is bounded, it is possible to construct this sequence as being bounded below. By the Stone-Weierstrass Theorem, for each $g_m$ there exists $\xi_m$ such that $\xi_m$ is a polynomial in $\R^n$ and $|g_m-\xi_m|<\frac{1}{m}$. Define a sequence $\{f_m\}$ with $f_m = \xi_m - \frac{1}{m}$. Thus for the sequence $\{f_m\}$ we have:
\begin{itemize}
    \item $f_m$ is $C^1$ ($C^\infty$ in fact).
    \item $f_m = \xi_m-\frac{1}{m} < g_m \leq f$.
    \item $\lim f_m = f$. Furthermore, for $x\in D$, $\lim f_m = \pxi$.
    \item There exists $k$ such that $||f_m||\leq k$ for all $m$.
\end{itemize}
That is, $\{f_m\}$ is a sequence of bounded $C^1$ under-approximations of $\pxi$ with a limit of $\pxi$.

Define $F_m(y) := \int_0^{y_i} f_m(y_{-i},t)dt$, so that $\frac{\partial F_m}{\partial x_i} = f_m$. For $y\in[0,x]$, consider $\l^{(y,i)}$. Since $\l^{(y,i)}$ is a straight line path that varies only in the $i^{th}$ component and has a velocity of 1, $\frac{d (F_m  \circ  \l^{(y,i)})}{d t}$ is analogous to the partial derivative of $F_m$ with respect to $i$.  Specifically, $\frac{d (F_m  \circ  \l^{(y,i)})}{d t} = (f_m  \circ  \l^{(y,i)})\text{sign}(x_i-x_i') = f_m  \circ  \l^{(y,i)}$. Similarly, on $D$ we have $\frac{d (F  \circ  \l^{(y,i)})}{d t} = (f  \circ  \l^{(y,i)})\text{sign}(x_i-x_i') = f  \circ  \l^{(y,i)}$.

Now, when $\l^{(y,i)}$ is on the region $D$ we have, 

$$\frac{d (F  \circ  \l^{(y,i)})}{d t} = \pxi  \circ  \l^{(y,i)} = f  \circ  \l^{(y,i)}> f_m  \circ  \l^{(y,i)} = \frac{d (F_m  \circ  \l^{(y,i)})}{d t} $$
where the inequality is gained because $\l^{(y,i)}$ is a strictly increasing function in this case, and $f>f_m$ by construction.

Note $F  \circ  \l^{(y,i)}$ is Lipschitz with Lipschitz constant $L$. If  $\frac{d (F  \circ  \l^{(y,i)})}{d t}$ exists on $\partial D$, then when $\l^{(y,i)}$ is on the region $\partial D$ we have

$$\frac{d (F  \circ  \l^{(y,i)})}{d t}\geq -L = f  \circ  \l^{(y,i)} > f_m  \circ  \l^{(y,i)} = \frac{d( F_m  \circ  \l^{(y,i)})}{d t}$$

This implies $\frac{d (F  \circ  \l^{(y,i)})}{d t} - \frac{d( F_m  \circ  \l^{(y,i)})}{d t}$ is non-negative where it exists.

Finally, $F  \circ  \l^{(y,i)}$ is Lipschitz, so it's derivative exists almost everywhere, and $\int_0^\a \frac{d (F  \circ  \l^{(y,i)})}{d t} = F  \circ  \l^{(y,i)} (\a) - F  \circ  \l^{(y,i)} (0)$. From this, we gain

$$\int_0^\a \frac{d (F  \circ  \l^{(y,i)})}{d t} - \frac{d( F_m  \circ  \l^{(y,i)})}{d t} dt = F  \circ  \l^{(y,i)}(\a) - F_m  \circ  \l^{(y,i)}(\a) + F  \circ  \l^{(y,i)}(0) - F_m  \circ  \l^{(y,i)}(0)$$

The above is the integral of a non-negative function, and is thus non-decreasing in $\a$. This implies that $F-F_m$ is non-decreasing from $0$ to $x$ in the $i^{th}$ component. So $A_i(x,x',F-F_m)\geq 0$ and $A_i(x,x',F)\geq A_i(x,x',F_m)$. Employing Theorem~\ref{thm2}, we have

\begin{align*}
    &A_i(x,x',F)\\
    \geq &\underset{m \rightarrow \infty }{\lim} A_i(x,F_m)\\
    = &\underset{m \rightarrow \infty }{\lim}\int_{\g\in\G^m(x)} A_i^\g(x,F_m) d\mu^x(\g)\\
    = & \underset{m \rightarrow \infty }{\lim} \int_{\g\in\G^m(x)} \int_0^1 f_m \frac{\partial \g_i}{\partial t} dt d\mu^x(\g)\\
    = & \int_{\g\in\G^m(x)} \int_0^1 \underset{m \rightarrow \infty }{\lim} f_m \frac{\partial \g_i}{\partial t} dt d\mu^x(\g)\\
    = & \int_{\g\in\G^m(x)} \int_0^1 \frac{\partial F}{\partial \g_i} \frac{\partial \g_i}{\partial t} dt d\mu^x(\g)\\
\end{align*}

We move the limit inside the integral by the dominated convergence theorem. We can move the limit inside the interior integral because $f_m$ is bounded, $\frac{\partial \g_i}{\partial t}$ is bounded using the constant velocity path parameterization, and the interior terms have a point-wise limit of $\frac{\partial F}{\partial \g_i} \frac{\partial \g_i}{\partial t}$ almost everywhere for almost every $\g$. To move the limit into the first integral, note that for particular values of $c_i$ we can employ Lemma~\ref{lemma3} to bound $A_i^\g(x,F_m +c_iy_i) = A_i^\g(x,F_m) + c_ix_i$ above or below zero. Thus $A_i^\g(x,F_m)$ has an upper and a lower bound. Using an over-approximating sequence for $\{f_m\}$ instead of an under-approximating sequence yields the same inequality in reverse, gaining our result.

\end{proof}

\section{Proof of Theorem~\ref{thm4}}\label{Proof:Thm4}
First, we begin the case were IG may fail to be Lipschitz. Consider $F(y_1,y_2) = \max(y_2-y_1,y_1-y_2)$. Let $\epsilon > 0$. Set $x' = (0,0)$ and consider $x = (1,1+\frac{\epsilon}{2})$, $\bar{x} = (1,1-\frac{\epsilon}{2})$. Let $\g, \bar{\g}$ be the $\I$ paths corresponding to $x,\bar{x}$, respectively.

First, note that $||x-\bar{x}|| = \epsilon$. We find that $\frac{\partial F}{\partial y_1} = 1$ if $y_1>y_2$, and $\frac{\partial F}{\partial y_1} = -1$ if $y_1<y_2$. So $\I_1(x,x',F) = (x_1-x_1')\int_0^1(-1) dt = 1$, while $\I_1(\bar{x},x',F) = -1$. So $|\I_1(x,x',F) - \I_1(\bar{x},x',F)| = 2$ for $\epsilon > 0$. Thus $\I(x,x',F)$ is not Lipschitz continuous in $x$.

Now we present the proof of the second claim:
\begin{proof}
Fix $x'$ and let $F$ be such that $\nabla F$ is Lipschitz continuous with constant $L$. Since $\nabla F$ is continuous on a bounded domain, $|\pxi|$ attains a maximum on $[a,b]$, call it $M$. Choose any $x,\bar{x} \in [a,b]$. We will denote the uniform-velocity paths for $x, \bar{x}$ by $\g, \bar{\g}$, respectively. Then,

\begin{align*}
    &|\I_i(x,x',F)-\I_i(\bar{x},x',F)|\\
    =&|(x_i-x_i')\int_0^1\pxi (\g(t))dt-(\bar{x}_i-x_i')\int_0^1\pxi (\bar{\g}(t))dt|\\
    =& |(x_i-\bar{x}_i)\int_0^1\pxi (\g(t))dt-(\bar{x}_i-x_i')\int_0^1[\pxi (\bar{\g}(t)) - \pxi (\g(t))]dt|\\
    \leq& |x_i-\bar{x}_i|\int_0^1|\pxi (\g(t))|dt+|\bar{x}_i-x_i'|\int_0^1|\pxi (\bar{\g}(t)) - \pxi (\g(t))|dt\\
    \leq& ||x-\bar{x}|| M + |b_i-a_i|\int_0^1L||\g(t)-\bar{\g}(t)|| dt\\
    =& ||x-\bar{x}|| M + |b_i-a_i|L\int_0^1||(x-\bar{x})t|| dt\\
    =& (M+\frac{|b_i-a_i|}{2}L)||x-\bar{x}||
    \end{align*}

Thus $\I_i(x,x',F)$ is Lipschitz continuous with Lipschitz constant at most $M+\frac{|b_i-a_i|}{2}L$.
\end{proof}

\section{Distributional IG Satisfies Distributional Attribution Axioms}\label{app:distributionalAxioms}

Here we provide proofs that distributional IG satisfies given axioms.

\begin{proof}[Sensitivity(a)]
Suppose $X'$ varies in exactly one input, $X_i'$, so that $X_j' = x_j$ for all $j\neq i$, and $\E F(X')\neq F(x)$. Then
\begin{align*}
    \EG_i(x,X',F) &= \E \I_i(x,X',F)\\
    &= \E (F(x) - F(X'))\\
    &= F(x) - \E F(X') \neq 0
\end{align*}
The second line is gained because IG satisfies completeness and $x_j = X_j$ causes $\I_j(x,X',F)=0$ for $j\neq i$.
\end{proof}

\begin{proof}[Completeness]
\begin{align*}
    \sum_{i=1}^n \EG(x,X',F) &= \sum_{i=1}^n \E \I(x,X',F)\\
    &= \E \sum_{i=1}^n \I(x,X',F)\\
    &= \E F(x)-F(X') & (\I \text{ satisfies completeness.})\\
    &= F(x) - \E F(X')
\end{align*}
\end{proof}

\begin{proof}[Symmetry Preserving]
Suppose that for all $x$, $F(x) = F(x^*)$, $X'_i$ and $X'_j$ are exchangeable, and $x_i = x_j$. Let $1_k$ represent a vector with every component 0 except the $k^{th}$ component, which is 1. First observe:
\begin{align*}
    \pxi(x) &= \lim_{t\rightarrow 0} \frac{F(x+1_it)-F(x)}{t}\\
    &= \lim_{t\rightarrow 0}\frac{F(x^*+1_jt)-F(x^*)}{t}\\
    &= \frac{\partial F}{\partial x_j}(x^*)
\end{align*}
From this, we have
\begin{align*}
    \EG_i(x,X',F) &= \E_{X'\sim D} (x_i-X_i')\int_0^1\pxi(X'+(x-X')t)dt\\
    &= \E_{X'\sim D} (x_i-X'_i)\int_0^1\frac{\partial F}{\partial x_j}((X'+(x-X')t)^*)dt\\
    &= \E_{X'\sim D} (x_j^*-X'^*_j)\int_0^1\frac{\partial F}{\partial x_j}(X'^*+(x^*-X'^*)t)dt \hfill & (x_i=x_j^*)\\
    &= \E_{X'^*\sim D} (x_j^*-X'^*_j)\int_0^1\frac{\partial F}{\partial x_j}(X'^*+(x^*-X'^*)t)dt \hfill & (X'^*\sim D\iff X'\sim D)\\
    &= \EG_j(x^*,X'^*,F)\\
    &= \EG_j(x,X',F) \hfill &(x=x^*; X', X'^*\sim D)\\
\end{align*}
\end{proof}

\begin{proof}[NDP]
Suppose $F$ is non-decreasing from every point on the support of $D$ to $x$. Then for any $x'$ on the support of $X'$, $\I(x,x',F)\geq 0$. Thus, for any $i$
\begin{align*}
    \EG_i(x,X',F) = \E\I_i(x,X',F)\geq 0
\end{align*}
\end{proof}

\section{Additional Experiments from \Cref{sec:exp}}\label{AppendixExperiments}
\subsection{Further ImageNet Results}

To further explore the pruning based on internal neuron attributions for image patches. We pick an often referenced image for IG, a fireboat, and repeat the experiments in \Cref{sec:bd-box}. The results are shown in \Cref{fig:FireBoatImg1,fig:FireBoatImg2}.

\begin{figure}[h]
    \begin{center}
        \begin{minipage}{0.2\linewidth}
            \centering
            \includegraphics[width=\linewidth]{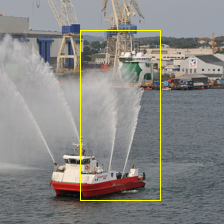}
        \end{minipage}
        \hspace{0.1in}
        \begin{minipage}{0.2\linewidth}
            \centering
            \includegraphics[width=\linewidth]{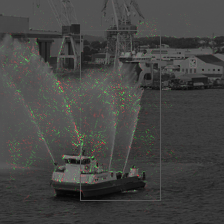}
        \end{minipage}
        \hspace{0.1in}
        \begin{minipage}{0.2\linewidth}
            \centering
            \includegraphics[width=\linewidth]{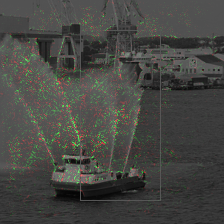}
        \end{minipage}
        \hspace{0.1in}
        \begin{minipage}{0.2\linewidth}
           \centering
           \includegraphics[width=\linewidth]{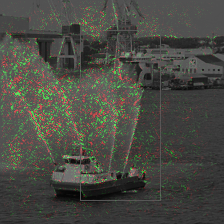}
        \end{minipage}
    \end{center}
    \caption{From left to right are images A,B,C,and D. A: The original image and bounding box indicating specified image patch. B: IG attributes visualized. Green dots show positive IG, red dots show negative IG. C: IG attributes visualized after top $1\%$ of neurons pruned based on image-patch attributions. D: IG attributes visualized after top $1\%$ neurons pruned based on the global ranking.}
    \label{fig:FireBoatImg1}
\end{figure}

\Cref{fig:FireBoatImg1} shows that pruning $1\%$ of the neurons based on targeted-pruning results in some scattering of activity, but the IG's focus on the leftmost water jets is still present. On the contrary, when pruned by global ranking, the IG is broadly scattered, and the focus on the leftmost water jest is diminished.
 
\begin{figure}[H]
    \centering
    \includegraphics{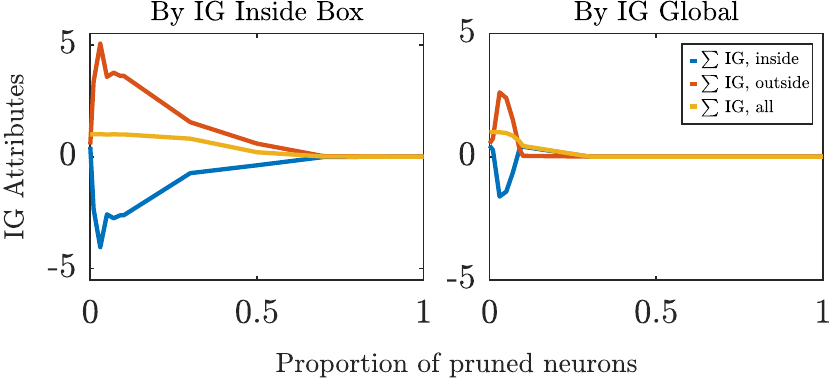}
    \caption{Sum of IG attributes inside and outside the bounding box when neurons are pruned according to certain rankings. Left: Neurons are pruned based on IG global ranking. Right: Neurons are pruned based on the IG ranking inside the bounding box. }
    \label{fig:FireBoatImg2}
\end{figure}

\Cref{fig:FireBoatImg2} reinforces the observations we have in \Cref{fig:TrafficLightExperiment}. We see that pruning by IG rankings inside the bounding box make the IG sum inside the box more negative and outside the box more positive compared to the pruning by global ranking. The observation again supports that image-patch based ranking gives higher ranks to the neurons that are responsible for positive IG inside the box.

\subsection{Fashion MNIST Results}

Here we present experiments on internal neuron attributions with a custom model trained on the Fashion MNIST data set. Information about the model can be found in the Appendix~\ref{MNISTModelInfo}.

In this experiment we identify a sub-feature common to each image in the category \textit{Sneaker}: the heel. We stipulate a bounding box for the heel, seen in Figure~\ref{fig:ShoesWithBox}. For each \textit{Sneaker} image we calculate image-patch attributions for neurons in the second to last dense layer. To calculate the each neuron's rank, we average it's attributions over all \textit{Sneaker} images. We then progressively prune based on rankings while noting the IG sums inside and outside the bounding box.

\begin{figure}[h]
    \centering
    \begin{minipage}{.2\linewidth}
        \centering
        \includegraphics[width=\linewidth]{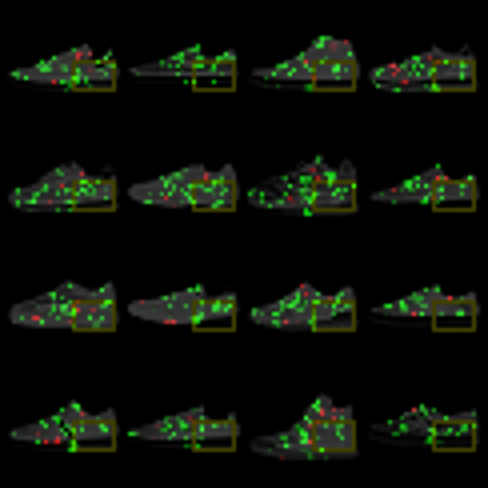}
    \end{minipage}%
    \hspace{0.05\linewidth}
    \begin{minipage}{.2\linewidth}
        \centering
        \includegraphics[width=\linewidth]{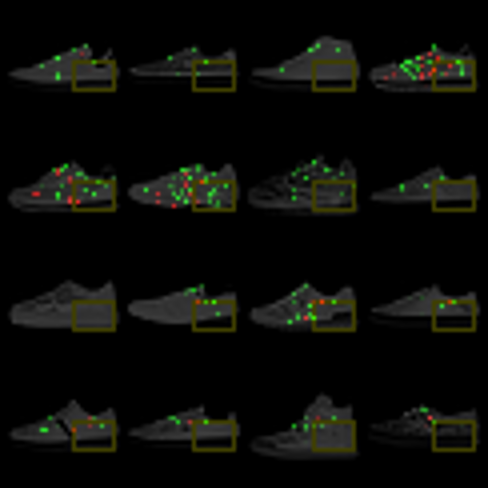}
    \end{minipage}%
    \caption{Left: IG attributes with respect to the sneaker images. Green dots show positive IG attributes and red dots show negative IG attributes. Bounding boxes are shown in yellow. Right: Recomputed IG attributes with respect to the sneaker images after internal neurons are pruned.}
    \label{fig:ShoesWithBox}     
\end{figure}

\begin{figure}[H]
    \centering
    \includegraphics{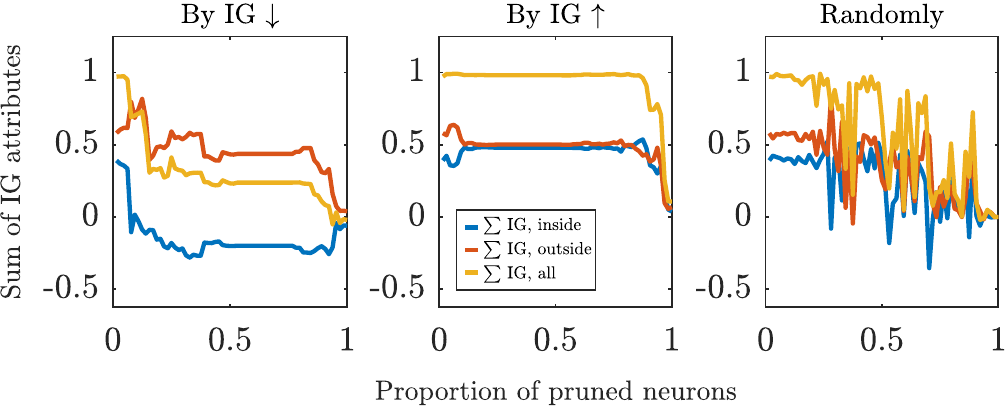}
    \caption{Summation of the recomputed IG attributes inside the bounding box, outside the bounding box and both. Summations are averaged over 64 samples chosen from testing set. Left: Pruning by the IG ranking in descending order. Middle: Pruning by the IG ranking in ascending order. Right: Randomly pruning.}
    \label{fig:my_label}
\end{figure}

When we prune in descending order, the average IG sum inside the box initially drops while the sum outside the box increases. A gap between the IG sums widens, and is sustained through the pruning process. This shows that the pruning targeted neurons that contributed to positive IG values in the box. Thus the regional IG accurately identified neurons postively associated with the heel region.

\section{Model Architecture and Training Parameters}\label{MNISTModelInfo}
Table~\ref{tab:net-arch} presents the architecture of the model used in the MNIST experiments.
\label{app:model-arch}

\begin{table}[H]
\centering
\begin{tabular}{@{}ll@{}}
\toprule
Layer Type           & Shape               \\ \midrule
Convolution $+$ tanh & $5 \times 5 \times 5$     \\ 
Max Pooling          & $2 \times 2$                \\ 
Convolution $+$ tanh & $5 \times 5 \times 10$ \\ 
Max Pooling          & $2 \times 2$             \\
Fully Connected $+$ tanh & $160$ \\ 
Fully Connected $+$ tanh & $64$ \\ 
Softmax & $10$ \\ \bottomrule
\end{tabular}
\caption{Model Architecture for the Fashion MNIST dataset}\label{tab:net-arch}
\end{table}

\end{document}